\documentclass[twoside]{article}

\usepackage[accepted]{aistats2022}
\usepackage[round]{natbib}

\usepackage{graphicx}
\usepackage{amsthm}
\usepackage{amsfonts}
\usepackage{amssymb}
\usepackage{bm}
\usepackage{bbm}
\usepackage[hidelinks]{hyperref}
\usepackage{enumitem}
\usepackage{amsmath}
\usepackage{float}

\newtheorem{theorem}{Theorem}
\newtheorem{proposition}{Proposition}
\newtheorem{lemma}{Lemma}
\graphicspath{ {./figures/} }

\begin{document}

%

%
\runningauthor{Benjamin Letham, Phillip Guan, Chase Tymms, Eytan Bakshy, Michael Shvartsman}

\twocolumn[

\aistatstitle{Look-Ahead Acquisition Functions for Bernoulli Level Set Estimation}


\aistatsauthor{ Benjamin Letham \And Phillip Guan \And  Chase Tymms}
\aistatsaddress{ Meta\\bletham@fb.com \And Reality Labs Research, Meta\\philguan@fb.com \And Reality Labs Research, Meta\\tymms@fb.com}
\vspace{-15pt}
\aistatsauthor{ Eytan Bakshy \And Michael Shvartsman}
\aistatsaddress{ Meta\\ebakshy@fb.com \And Reality Labs Research, Meta\\michael.shvartsman@fb.com}

]

\begin{abstract}
Level set estimation (LSE) is the problem of identifying regions where an unknown function takes values above or below a specified threshold. Active sampling strategies for efficient LSE have primarily been studied in continuous-valued functions. Motivated by applications in human psychophysics where common experimental designs produce binary responses, we study LSE active sampling with Bernoulli outcomes. With Gaussian process classification surrogate models, the look-ahead model posteriors used by state-of-the-art continuous-output methods are intractable. However, we derive analytic expressions for look-ahead posteriors of sublevel set membership, and show how these lead to analytic expressions for a class of look-ahead LSE acquisition functions, including information-based methods. Benchmark experiments show the importance of considering the \textit{global} look-ahead impact on the entire posterior. We demonstrate a clear benefit to using this new class of acquisition functions on benchmark problems, and on a challenging real-world task of estimating a high-dimensional contrast sensitivity function.
\end{abstract}


\section{INTRODUCTION}
The \textit{level set estimation} (LSE) problem is to identify the regions where a black-box function $f(\mathbf{x})$ is above or below a particular threshold $\gamma$. Applications include environmental monitoring (\textit{e.g.} identifying areas where a contaminant is at a hazardous level, \citealp{lse}), communications (\textit{e.g.} finding wireless network configurations with acceptable signal quality, \citealp{lse_comms}), and finance (\textit{e.g.} for derivative pricing, \citealp{lyu21}). Evaluating $f(\mathbf{x})$ in these settings entails a time-consuming physical measurement or computer simulation, so the goal is to identify the level set with as few samples as possible.

Sample-efficient LSE is an active learning problem \citep{kapoor07, hoang14}, and is related to techniques like Bayesian active learning by disagreement \citep[BALD,][]{Houlsby2011} that use a surrogate model to perform active sampling. The function $f(\mathbf{x})$ is modeled with a Gaussian process (GP) surrogate, and then active sampling is driven by an acquisition function that selects the most valuable point to sample for identifying the level set. Several acquisition functions have been developed for LSE, as described in Section \ref{sec:background}, primarily for continuous-output functions with Gaussian noise.

The motivation for this work comes from psychophysics, a field of science that seeks to understand perception of physical stimuli \citep{fechner1966elements}. Understanding and adjusting for the limitations of human perception is important for a variety of downstream applications including audio/visual compression \citep{pappas1996supra, nadenau2000human}, hearing aid design \citep{moore1996perceptual}, clinical evaluation of auditory and visual impairments \citep{de1975clinical, fitzke1988clinical}, and the design of virtual and augmented reality systems \citep{kress2020optical}.

A common task in psychophysics is to identify \textit{detection thresholds}, the smallest stimulus intensity (\textit{e.g.}, volume of a sound) at which the stimulus can be perceived, usually as a function of other stimulus properties (\textit{e.g.}, frequency). Finding detection thresholds is an LSE problem for the threshold level, notably with Bernoulli observations $y \in \{0, 1\}$ indicating whether or not a stimulus was perceived. GP surrogate models and active sampling methods have been applied to psychophysics threshold estimation problems with Bernoulli responses \citep{Gardner2015a, song2015, Song2017b, Song2018, cox2016bayesian, Schlittenlacher2018, Schlittenlacher2020, owen2021}. This prior work has been limited to a 1-d or 2-d stimulus space of audiograms and, in contrast to our work here, has not used LSE acquisition functions, rather has used global learning methods such as BALD and then extracted the level set \textit{post hoc} from the surrogate posterior. As a notable exception, \citet{owen2021} used the straddle acquisition, which is described below.

\textit{Look-ahead} acquisition functions select points based on the impact their observation will have on the subsequent surrogate model. They are the state-of-the-art approach for LSE \citep{truvar, lyu21, nguyen21}, as well as for Bayesian optimization (BO) \citep{scott11kg, hernandezlobato2014pes, wang2017maxvalue, botorch}, which is related to LSE by the use of GP surrogates and acquisition functions. They have additionally been shown to be useful in the psychophysics setting, albeit with a parametric model and only in one dimension \citep{Kim2017}. 

Look-ahead acquisition functions rely on being able to compute a posterior update given the proposed observation, which can be done analytically with Gaussian observations. Unfortunately, with Bernoulli observations the surrogate model posterior updates are no longer analytic. This is vital for acquisition optimization, which requires computing many thousands of look-ahead posteriors throughout the course of active sampling.

Our work here enables tractable look-ahead acquisition for the Bernoulli LSE problem. Specifically, the contributions of this paper are:
\begin{enumerate}[leftmargin=*]
    \setlength\itemsep{0em}
    \item Despite the look-ahead surrogate model posterior being intractable with Bernoulli observations, we derive posterior update formulae that, remarkably, enable exact, closed-form computation of several state-of-the-art look-ahead acquisition functions, including information-based approaches.
    \item Our look-ahead posterior formulae enable easy construction of novel acquisition functions, which we show by introducing \textit{expected absolute volume change} (EAVC), a new acquisition function inspired by max-value entropy search in BO. 
    \item We evaluate the acquisition functions with a thorough simulation study, which shows that look-ahead is critical for achieving good level set estimates in high dimensions. We also show that simply being look-ahead is not enough to ensure reliable performance---the acquisition function must also be \textit{global}, a distinction we discuss in Section \ref{sec:background}.
    \item Our work enables rapid acquisition function computation that is suitable for human experiments, which we show by applying our global look-ahead acquisition functions to a real, high-dimensional psychophysics problem.
\end{enumerate}

\section{BACKGROUND}\label{sec:background}

\subsection{Models for Level Set Estimation}
We consider a black-box function $f(\mathbf{x})$, with $\mathbf{x} \in \mathcal{B}$ and $\mathcal{B} \subset \mathbb{R}^d$ a compact set. Our goal is to identify the set $\mathcal{L}_{\gamma}(f) = \{\bm{x}: f(\bm{x}) \leq \gamma\}$, known as the \textit{sublevel set}. When $f$ has continuous outputs, it is typical to assume a Gaussian observation model $y = f(\mathbf{x}) + \epsilon$, with
$\epsilon$ i.i.d. Gaussian noise. We give $f$ a GP prior, then, given a set of observations $\mathcal{D}_n = \{(\bm{x}^\textrm{obs}_i, y^\textrm{obs}_i)\}_{i=1}^n$, the joint posterior
for any set of points is a multivariate normal (MVN) with analytic mean and covariance \citep{williams2006gaussian}. We denote the marginal posterior at $\mathbf{x}$ as $f(\mathbf{x}) | \mathcal{D}_n \sim \mathcal{N}(\mu(\mathbf{x} | \mathcal{D}_n), \sigma^2(\mathbf{x}|\mathcal{D}_n))$.

With Bernoulli observations $y \in \{0, 1\}$, the standard practice is to use a classification GP based on either a logit or probit model \citep{kuss2005}. Here we focus on the probit case, in which $y \sim \textrm{Bernoulli}(z(\mathbf{x}))$, with latent probability $z(\mathbf{x}) = \Phi(f(\mathbf{x}))$ and $\Phi(\cdot)$ denoting the Gaussian cumulative distribution function. With Bernoulli observations, the predictive posterior for $f | \mathcal{D}_n$ requires approximation, but a variety of efficient approximations have been developed, including Laplace approximation \citep{gpclass_lp}, expectation propagation \citep{gpclass_ep}, and variational inference \citep{gpclass_vi}. Our experiments use variational inference, but the acquisition functions we develop are agnostic to how inference is done---for our purposes it is sufficient to have an MVN posterior for $f$.

We define $\theta = \Phi(\gamma)$ to be the desired threshold for the probability function $z(\mathbf{x})$. Then, $\mathcal{L}_{\gamma}(f) = \mathcal{L}_{\theta}(z)$, and so LSE  for $z(\mathbf{x})$ can equivalently be done directly on $z(\mathbf{x})$ or in the latent space of $f(\mathbf{x})$. An important quantity for the acquisition functions described below is $\pi(\mathbf{x} | \mathcal{D}_n) = \mathbb{P}(\mathbf{x} \in \mathcal{L}_{\gamma}(f) | \mathcal{D}_n)$, which we call the \textit{level set posterior}. Given the GP posterior, we can compute the level set posterior as
\begin{equation}\label{eq:levelsetpost}
    \pi(\mathbf{x} | \mathcal{D}_n) = \Phi \left( \frac{\gamma - \mu(\mathbf{x} | \mathcal{D}_n)}{\sigma(\mathbf{x} | \mathcal{D}_n)} \right).
\end{equation}

\subsection{Acquisition for Level Set Estimation}

\subsubsection{Non-look-ahead Acquisition}
During active sampling, at each iteration we select a maximizer of the acquisition function to be the next point sampled.
General purpose active sampling strategies such as BALD seek to reduce global uncertainty in the posterior of $f$ or $z$, which can waste samples by reducing variance in regions that are far from the threshold, the main area of interest for LSE.

\citet{bryan2006} developed the first acquisition functions tailored for LSE, the most successful of which was the \textit{straddle}, which is of similar flavor to the well-known upper confidence bound (UCB) acquisition function \citep{gpucb}:
\begin{equation*}
    \alpha_{\textrm{straddle}}(\mathbf{x}_*) = -|\mu(\mathbf{x}_* | \mathcal{D}_n) - \gamma | + \beta \sigma(\mathbf{x}_* | \mathcal{D}_n).
\end{equation*}
As in UCB, the parameter $\beta$ drives exploration by selecting higher variance points; \citet{bryan2006} used $\beta = 1.96$. They also considered as acquisition functions the misclassification probability and the entropy of the level set posterior:
\begin{align*}
    \alpha_{\textrm{misclass}}(\mathbf{x}_*) &= \min(\pi(\mathbf{x}_* | \mathcal{D}_n), 1-\pi(\mathbf{x}_* | \mathcal{D}_n)), \\
    \alpha_{\textrm{entropy}}(\mathbf{x}_*) &= H_b(\pi(\mathbf{x}_* | \mathcal{D}_n)),
\end{align*}
where $H_b(p) = -p \log_2 p - (1 - p) \log_2 (1 - p)$ is the binary entropy function. They found that the straddle acquisition performed best, though noted that all of these acquisition functions are ``subject to oversampling edge positions." 
\cite{lse} provided theoretical grounding for LSE  and proved sample complexity bounds for extensions of the straddle.
\citet{ranjan08} adapted the Expected Improvement criterion from BO \citep{jones98} by defining an improvement function with respect to the threshold.

\subsection{Look-ahead Acquisition}\label{sec:lookahead}
Subsequent development of acquisition functions for LSE  have focused on look-ahead approaches, which consider not just the posterior at the candidate point $\mathbf{x}_*$, but how the posterior at a different point $\mathbf{x}_q$ will change as a result of an observation at $\mathbf{x}_*$. We denote the look-ahead dataset as $\mathcal{D}_{n+1}(\mathbf{x}_*, y_*) = \mathcal{D}_n \cup \{(\mathbf{x}_*, y_*)\}$, and note that $\mathcal{D}_{n+1}(\mathbf{x}_*, y_*)$ is a random variable via $y_*$. Much of the past work in look-ahead acquisition has relied on the useful GP property that, \textit{with Gaussian observations}, the look-ahead variance $\sigma^2(\mathbf{x} | \mathcal{D}_{n+1}(\mathbf{x}_*, y_*))$ does not depend on $y_*$, and can be computed analytically.

\citet{picheny10} introduced the first look-ahead acquisition function for LSE , \textit{targeted integrated mean squared error} (tIMSE), which minimizes look-ahead posterior variance, weighted according to distance to the threshold by some function $w_{\gamma}(\mathbf{x})$:
\begin{align}\nonumber
    \alpha_{\textrm{tIMSE}}(&\mathbf{x_*}) = - \int_{\mathcal{B}}  \sigma^2(\mathbf{x} | \mathcal{D}_{n+1}(\mathbf{x}_*, y_*)) w_{\gamma}(\mathbf{x}) d\mathbf{x}\\\label{eq:timse}
    &\approx - C \sum_{\mathbf{x}_q \in \mathcal{G}} \sigma^2(\mathbf{x}_q | \mathcal{D}_{n+1}(\mathbf{x}_*, y_*)) w_{\gamma}(\mathbf{x}_q).
\end{align}
This is an example of a \textit{global} look-ahead acquisition function that evaluates the impact of observing $\mathbf{x_*}$ on the entire design space, using quasi-Monte Carlo (qMC) integration \citep{caflisch98}. Here $\mathcal{G}$ is a quasi-random sequence and $C = \frac{\textrm{Vol}(\mathcal{B})}{|\mathcal{G}|}$ is a constant that can be ignored for the purpose of acquisition optimization. The tIMSE formula in (\ref{eq:timse}) is analytic due to the analytic form of $\sigma^2(\mathbf{x}_q | \mathcal{D}_{n+1}(\mathbf{x}_*, y_*))$, and so can be cheaply evaluated in batch across a large set of global reference points $\mathcal{G}$. \citet{truvar} used a similar approach by minimizing the total look-ahead posterior variance of the region that was not yet decidedly classified as above or below threshold. \citet{zanette18} utilized the analytic look-ahead posterior to construct an expected improvement-based criterion.

An important class of methods are of the form
\begin{equation}\label{eq:general}
    \alpha(\mathbf{x_*}) = Q(\mathcal{D}_n) - \mathbb{E}_{y_*} [Q( \mathcal{D}_{n+1}(\mathbf{x}_*, y_*))]
\end{equation}
where $Q(\mathcal{D}_n)$ is a cost function applied to the surrogate model posterior, and $Q( \mathcal{D}_{n+1}(\mathbf{x}_*, y_*))$ is that same cost function applied to the look-ahead posterior. \textit{Stepwise uncertainty reduction} \citep[SUR,][]{bect12, chevalier14} uses expected classification error as the cost function, and thus maximizes the expected look-ahead misclassification error reduction:
\begin{align}\label{eq:QGlobalSUR}
    Q_\textrm{GlobalSUR}(\mathcal{D}_n) &= \sum_{\mathbf{x}_q \in \mathcal{G}} \min ( \pi(\mathbf{x}_q | \mathcal{D}_n), 1- \pi(\mathbf{x}_q | \mathcal{D}_n )) \\\label{eq:QGlobalSUR2}
    &= \sum_{\mathbf{x}_q \in \mathcal{G}}  \Phi \left( - \frac{|\mu(\mathbf{x}_q|\mathcal{D}_n) - \gamma|}{\sigma(\mathbf{x}_q|\mathcal{D}_n)} \right),
\end{align}
where the global impact is again being estimated via a sum over a quasi-random sequence, as in (\ref{eq:timse}). $Q_\textrm{GlobalSUR}(\mathcal{D}_{n+1}(\mathbf{x}_*, y_*))$ is computed in the same manner using the look-ahead posterior, so when (\ref{eq:QGlobalSUR2}) is plugged into (\ref{eq:general}), the acquisition function is analytic.

Acquisitions of the form (\ref{eq:general}) can be formed either as a global look-ahead, as with GlobalSUR, or as a \textit{localized} version that considers the look-ahead impact of observing $\mathbf{x}_*$ just on $\mathbf{x}_*$:
\begin{equation}\label{eq:QLocalSUR}
    Q_{\textrm{LocalSUR}}(\mathcal{D}_n) = \min ( \pi(\mathbf{x}_* | \mathcal{D}_n), 1- \pi(\mathbf{x}_* | \mathcal{D}_n )).
\end{equation}
\citet{lyu21} call this method ``gradient SUR." It avoids the summation over the global reference set $\mathcal{G}$ required to compute GlobalSUR, but provides a less direct measurement of the total value of observing $\mathbf{x}_*$. 

\citet{azzimonti21} applied the SUR strategy to the volume of misclassified points, via the Vorob'ev expectation, with particular focus on active learning of conservative level set estimates that seek to control the type I error rate. In Section \ref{sec:eavc} we also use measures of level set volume to construct an acquisition function, though with a different approach that does not control error rates and is not a SUR strategy.

Global information gain has long been a target for level set acquisition. \citet{bryan2006} wrote of their acquisition functions, ``we believe the good performance of the evaluation metrics proposed below stems from their being heuristic proxies for global information gain." \citet{nguyen21} constructed a localized mutual information (MI) acquisition function by taking
\begin{equation}\label{eq:QLocalMI}
    Q_{\textrm{LocalMI}}(\mathcal{D}_n) = H_b( \pi(\mathbf{x}_* | \mathcal{D}_n) ) 
\end{equation}
in the acquisition form (\ref{eq:general}), and called this strategy binary entropy search (BES). As with SUR, computation of the look-ahead term $Q_\textrm{LocalMI}(\mathcal{D}_{n+1}(\mathbf{x}_*, y_*))$ has hitherto relied on the analytic look-ahead posteriors that exist only for Gaussian observations. A criterion for global information gain naturally parallels GlobalSUR:
\begin{equation}\label{eq:QGlobalMI}
    Q_{\textrm{GlobalMI}}(\mathcal{D}_n) = \sum_{\mathbf{x}_q \in \mathcal{G}}  H_b( \pi(\mathbf{x}_q | \mathcal{D}_n) ). 
\end{equation}

The only work that has applied LSE acquisition functions to GP classification surrogates is that of \citet{lyu21}, who studied LSE  under high levels of heavy-tailed noise. For the Bernoulli likelihood, they derived an approximation for the intractable look-ahead variance $\sigma^2(\mathbf{x} | \mathcal{D}_{n+1}(\mathbf{x}_*, y_*))$ based on plug-in estimates, and further approximated $\mu(\mathbf{x} | \mathcal{D}_{n+1}(\mathbf{x}_*, y_*)) \approx \mu(\mathbf{x} | \mathcal{D}_{n})$, which enabled approximate SUR computation via (\ref{eq:QGlobalSUR2}).

\section{BERNOULLI LOOK-AHEAD POSTERIOR UPDATES}

With Bernoulli observations, the latent look-ahead posterior $p(f | \mathcal{D}_{n+1}(\mathbf{x}_*, y_*))$ is not analytic. However, other quantities important for computing acquisition functions are. The moments of the posterior $p(z | \mathcal{D}_n)$ are analytic:

\begin{proposition}\label{prop:zpost}
Let $a_* = \frac{\mu(\mathbf{x}_* | \mathcal{D}_n)}{\sqrt{1 + \sigma^2(\mathbf{x}_* | \mathcal{D}_n)}}$ and $c_* = \frac{1}{\sqrt{1 + 2\sigma^2(\mathbf{x}_* | \mathcal{D}_n)}}$. Then,
\begin{align*}
    \mathbb{E}[z(\mathbf{x}_*) | \mathcal{D}_n] &= \Phi (a_*),\\
    \textrm{Var}[z(\mathbf{x}_*) | \mathcal{D}_n] &= \Phi (a_*) - \Phi (a_*)^2 - 2T\left(a_*, c_* \right).
\end{align*}
\end{proposition}
Here $T(\cdot, \cdot)$ is Owen's T function, which can be computed efficiently \citep{owens_t} and is available in SciPy \citep{scipy}. The formula for the mean is well-known and used in several applications of classification GPs \citep[\textit{e.g.}][]{Houlsby2011}; the variance is derived in the supplement. This result enables computing the straddle acquisition on the posterior of $z$ instead of that of $f$, which accounts for the variance squashing of the $\Phi(\cdot)$ transformation. However, our primary interest lies in look-ahead posteriors, and $p(z | \mathcal{D}_{n+1}(\mathbf{x}_*, y_*))$ is intractable.

Our main result is that despite the look-ahead posteriors for $f$ and $z$ being intractable, the look-ahead level set posterior $\pi(\mathbf{x}_q | \mathcal{D}_{n+1}(\mathbf{x}_*, y_*))$ is analytic, in terms of $\Phi(\cdot)$ and $\textrm{BvN}(\cdot, \cdot; \rho)$, which denotes the standard (zero-mean, unit-variance) bivariate normal distribution function with correlation $\rho$.

\begin{theorem}\label{thm:post}
Let $b_q = \frac{\gamma - \mu(\mathbf{x}_q|\mathcal{D}_n)}{\sigma(\mathbf{x}_q|\mathcal{D}_n)}$,  $\sigma(\mathbf{x}_q, \mathbf{x}_* | \mathcal{D}_n) = \textrm{Cov}[f(\mathbf{x}_q), f(\mathbf{x}_*) | \mathcal{D}_n]$, and
\begin{equation*}
    Z_{q*} = \textrm{BvN} \left( a_*  ,  b_q; \frac{- \sigma(\mathbf{x}_q, \mathbf{x}_* | \mathcal{D}_n)}{\sigma(\mathbf{x}_q|\mathcal{D}_n) \sqrt{1 + \sigma^2(\mathbf{x}_*|\mathcal{D}_n)}} \right).
\end{equation*}
The level set posterior at $\mathbf{x}_q$ given observation $y_*=1$ at $\mathbf{x}_*$ is
\begin{equation*}
\pi(\mathbf{x}_q | \mathcal{D}_{n+1}(\mathbf{x}_*, y_* = 1)) = \frac{Z_{q*}}{\Phi \left(a_* \right)}.
\end{equation*}
Given observation $y_*=0$, the level set posterior at $\mathbf{x}_q$ is
\begin{equation*}
\pi(\mathbf{x}_q | \mathcal{D}_{n+1}(\mathbf{x}_*, y_* = 0)) = \frac{\Phi\left( b_q \right) - Z_{q*}}{\Phi \left(-a_* \right)}.
\end{equation*}
If $\mathbf{x}_q = \mathbf{x}_*$, these results hold with $\sigma(\mathbf{x}_q, \mathbf{x}_* | \mathcal{D}_n) = \sigma^2(\mathbf{x}_q | \mathcal{D}_n) = \sigma^2(\mathbf{x}_* | \mathcal{D}_n)$ and $\mu(\mathbf{x}_q | \mathcal{D}_n) = \mu(\mathbf{x}_* | \mathcal{D}_n)$.
\end{theorem}
The proof is in the supplement. There are several routines for efficiently computing $\textrm{BvN}(\cdot, \cdot; \rho)$---we use the method of \citet{genz04}, which produces a differentiable estimate that enables us to compute gradients of the look-ahead level set posterior for acquisition optimization.

This result shows that the look-ahead level set posterior at $\mathbf{x}_q$ given an observation at $\mathbf{x}_*$ can be computed analytically using only the GP posterior at those points. Fig. \ref{fig:posteriors} shows an example of the global look-ahead posterior for the 2-d discrimination test problem from Section \ref{sec:experiments}. The posteriors are conditional on the binary outcome $y_*$, whose distribution is:
\begin{proposition}\label{prop:py1}
$\mathbb{P}(y_* = 1 | \mathcal{D}_n, \mathbf{x}_*) = \Phi(a_*).$ 
\end{proposition}

We now show how these results can be used to construct Bernoulli LSE versions of look-ahead acquisition functions from Section \ref{sec:lookahead}, as well as novel acquisition functions.

 \begin{figure*}[t]
     \centering
     \includegraphics{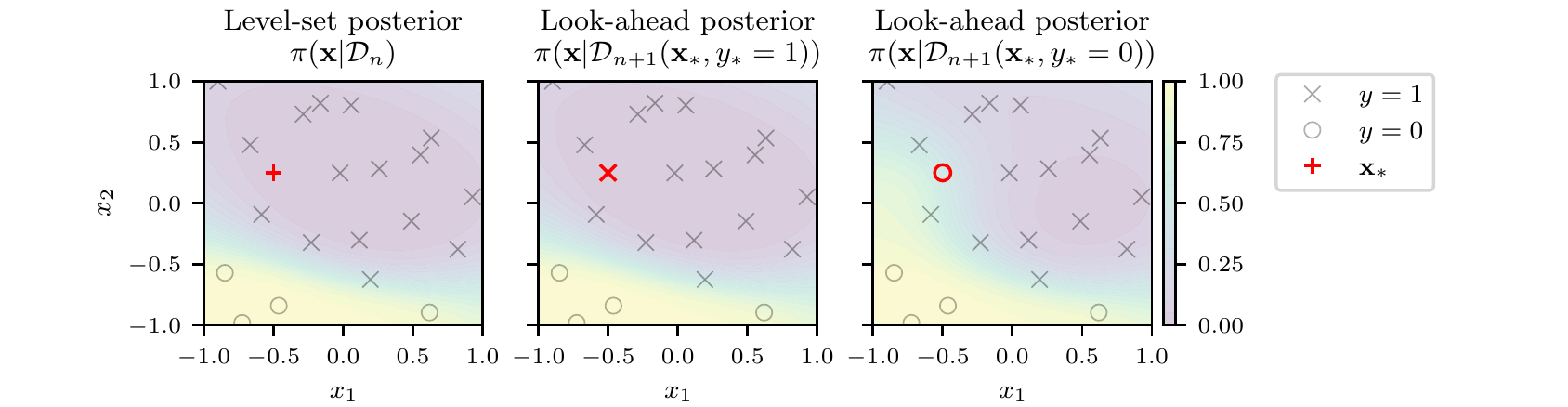}
     \vspace{-15pt}
     \caption{
     \textbf{Level set posteriors}.
     \textit{Left}: The level set posterior for a model fit to 20 observations (gray markers) of the 2-d discrimination test function. \textit{Middle, Right}: Look-ahead posteriors given an observation at $\mathbf{x}_*$ (red marker) of $y_*=1$ (Middle) and $y_*=0$ (Right). The look-ahead posteriors are computed analytically using the formulae of Theorem \ref{thm:post}, and form the basis of acquisition function computation.
     }
     \label{fig:posteriors}
     \vspace{5pt}
 \end{figure*}

 \begin{figure*}[t]
     \centering
     \includegraphics{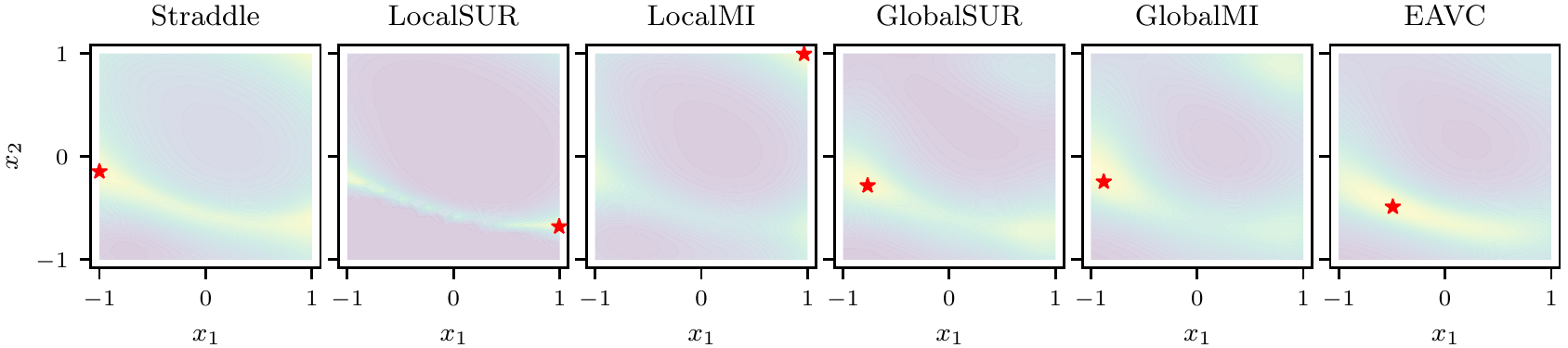}
     \vspace{-15pt}
     \caption{
     \textbf{Acquisition functions}. Acquisition surfaces for the same posterior in Fig. \ref{fig:posteriors}. The red star indicates the point that maximizes the acquisition function. Localized look-ahead methods (LocalSUR, LocalMI) show the same edge-seeking as the non-look-ahead Straddle. Global look-ahead methods (GlobalSUR, GlobalMI, EAVC) select interior points near the threshold.
     }
     \label{fig:acquisitions}
     \vspace{5pt}
 \end{figure*}

\section{BERNOULLI LOOK-AHEAD ACQUISITION FUNCTIONS}

The formulae in Theorem \ref{thm:post} enable efficient computation of any look-ahead acquisition function that depends on the model only via the level set posterior. In particular, acquisition functions of the form (\ref{eq:general}) can be computed for any $Q(\cdot)$ that is a function of $\pi(\cdot)$, which includes both SUR and MI. The key quantity is the expectation over the look-ahead posteriors, $\mathbb{E}_{y_*} [Q(\mathcal{D}_{n+1}(\mathbf{x}_*, y_*))]$, which can be computed via Theorem \ref{thm:post} and Proposition \ref{prop:py1}. For instance,
\begin{align*}
    \mathbb{E}_{y_*} [&Q_\textrm{LocalMI}(\mathcal{D}_{n+1}(\mathbf{x}_*, y_*))] = \\
    &\Phi(a_*) H_b \left( \frac{Z_{q*}}{\Phi \left(a_* \right)} \right) + \Phi(-a_*) H_b \left( \frac{\Phi\left( b_q \right) - Z_{q*}}{\Phi \left(-a_* \right)} \right).
\end{align*}
Expressions for GlobalSUR, LocalSUR, LocalMI, and GlobalMI in the Bernoulli case are obtained by plugging the posterior formulae into (\ref{eq:QGlobalSUR}), (\ref{eq:QLocalSUR}), (\ref{eq:QLocalMI}), and (\ref{eq:QGlobalMI}), respectively. These expressions are fully analytic; complete expressions for each acquisition function are given in the supplement.

\subsection{A Novel Volume Acquisition Function}\label{sec:eavc}
Besides SUR and MI, other quantities of interest for acquisition functions can be computed with the formulae in Theorem \ref{thm:post}. In BO, the successful max-value entropy search acquisition function \citep{wang2017maxvalue} finds the point that is most informative about the best \textit{function value} as opposed to being informative about the \textit{best point}. The same concept can be applied to LSE  by seeking points that are informative about the \textit{volume} of the sublevel set $\mathcal{L}_{\gamma}(f)$. The qMC expected sublevel-set volume is
\begin{equation}\label{eq:vol}
    \widetilde{V}(\mathcal{D}_n) = C \sum_{\mathbf{x}_q \in \mathcal{G}} \pi(\mathbf{x}_q| \mathcal{D}_n).
\end{equation}
There are many approaches one might take to assess how informative a candidate point $\mathbf{x}_*$ is about $\textrm{Vol}(\mathcal{L}_{\gamma}(f))$, and here we consider the expected absolute volume change (EAVC) produced by observing $\mathbf{x}_*$, which is a direct measure of how sensitive $\textrm{Vol}(\mathcal{L}_{\gamma}(f))$ is to the outcome at $\mathbf{x}_*$:
\begin{equation*}
\alpha_{\textrm{EAVC}}(\mathbf{x}_*) = \mathbb{E}_{y_*} \left[ \left| \widetilde{V}(\mathcal{D}_n) - \widetilde{V}(\mathcal{D}_{n+1}(\mathbf{x}_*, y_*)) \right| \right ].
\end{equation*}
The look-ahead volumes under $y_*=0$ and $y_*=1$ can be computed by plugging the look-ahead posteriors into (\ref{eq:vol}), producing an analytic acquisition function; the complete expression is given in the supplement. Along with SUR and MI, EAVC shows the breadth of the acquisition functions that can be computed using Theorem \ref{thm:post}. The acquisition functions cover a broad set of target criteria (misclassification error, entropy, and volume), and also have variety in their functional form: SUR and MI are both of the form (\ref{eq:general}), while EAVC is not, showing that acquisition functions do not have to be of the form (\ref{eq:general}) in order to be computed via Theorem \ref{thm:post}, or to be useful.

Fig. \ref{fig:acquisitions} shows each acquisition function when computed on the posterior of Fig. \ref{fig:posteriors}. The acquisition functions are broadly similar, with elevated values along the threshold and in the high-uncertainty region of the top-right corner. However, they have substantially different maxima, and thus propose different candidates for the next iteration. Straddle selects the point on the threshold at the edge of the design space, consistent with the observation of \citet{bryan2006} that it oversamples the edges. The localized look-ahead methods (LocalMI, LocalSUR) also select points on the edges. Edge points have high uncertainty in GP posteriors, and edge samples are highly informative about the edge point itself, the criterion for a localized look-ahead method. However, edge points are less informative about the global surface as a whole, and so we see the global look-ahead methods (EAVC, GlobalMI, GlobalSUR) select interior points along the highest-uncertainty portion of the threshold.

\begin{figure*}[tbh]
    \includegraphics{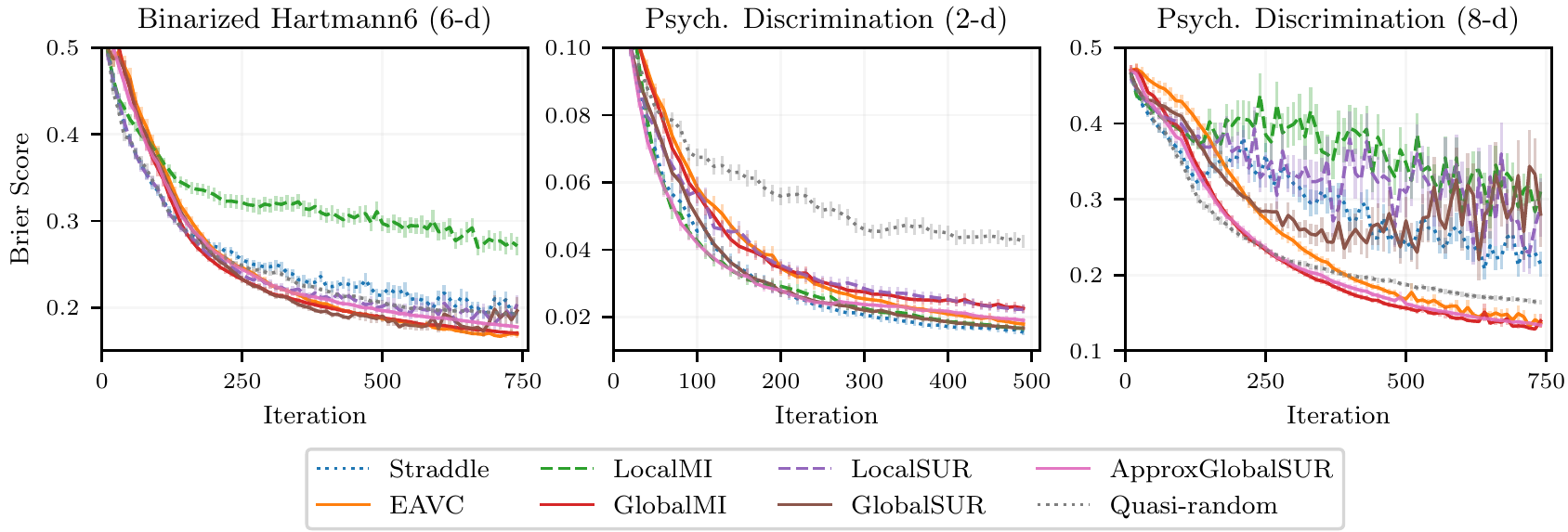}
    \vspace{-15pt}
    \caption{\textbf{Benchmark results}. Brier score (lower is better) for the level set posterior as a function of active sampling iteration, averaged over 280 repeated runs with error bars showing two standard errors. On the high-dimensional problems, straddle and the localized look-ahead methods (LocalMI, LocalSUR) did not perform better than the quasi-random baseline. Global methods GlobalMI and EAVC were best in high dimensions.
    }
    \label{fig:results-sim}
\end{figure*}

\section{BENCHMARK EXPERIMENTS}\label{sec:experiments}

\begin{figure*}[tbh]
    \centering
    \includegraphics[width=0.25\linewidth, height=0.2\linewidth]{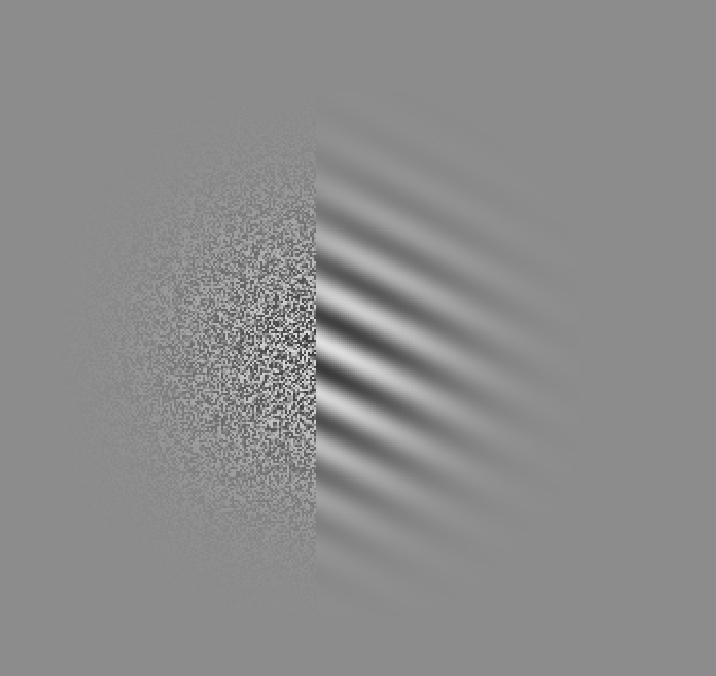}
    \includegraphics[width=0.25\linewidth, height=0.2\linewidth]{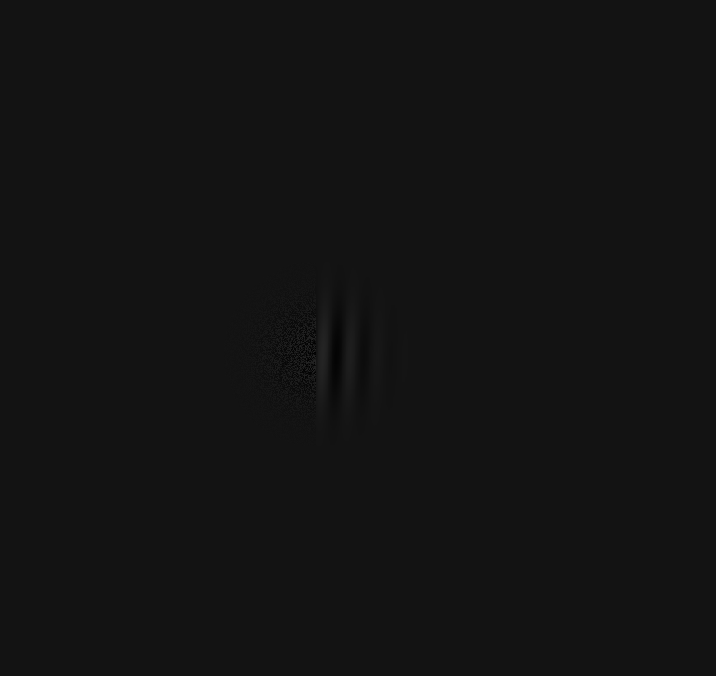}
    \includegraphics[width=0.25\linewidth, height=0.2\linewidth]{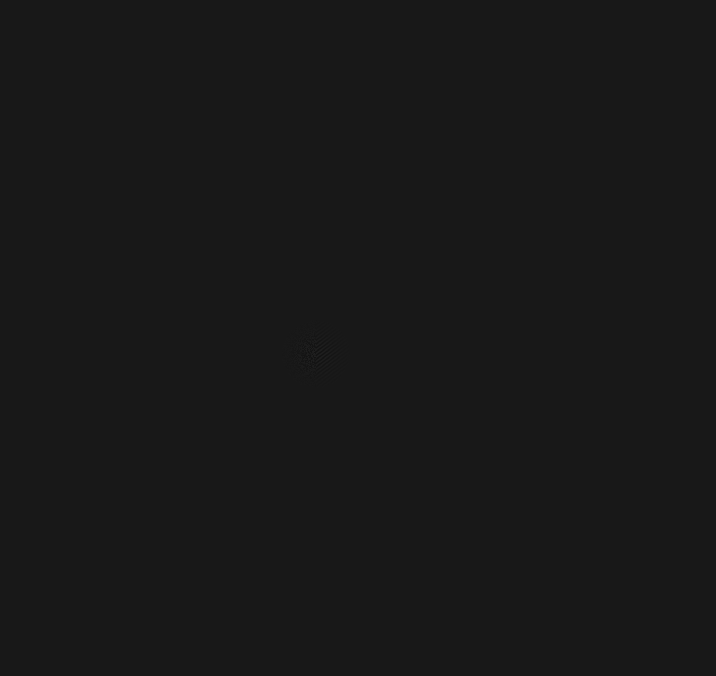}

    \caption{\textbf{Example stimuli from the real psychophysical discrimination task}. The human's task is to determine whether the white noise is on the right or the left of the image. In the three examples in this figure, the correct response is ``left.''
    \emph{Left}: A stimulus whose discrimination probability is approximately 1. The stimulus is both large and clearly visible against the background.  
    \emph{Middle}: A stimulus near the detection threshold of $p=0.75$. The contrast is very low and the stimulus is small.
    \emph{Right}: A stimulus whose discrimination probability is $p=0.5$. The stimulus is essentially invisible against the background and the participant must resort to random guessing.
    The real stimulus was additionally animated with some temporal frequency, and appeared at some distance and angle from the center of the screen.}
    \label{fig:real_stims}
\end{figure*}

We use three benchmark problems to evaluate and understand the performance of look-ahead acquisition functions for Bernoulli LSE. The first is a binarized version of the classic Hartmann 6-d function, using the same modified version of \citet{lyu21}, plus an affine transformation and an inverse probit transform to produce Bernoulli responses; see the supplement for the full functional form.

Inspired by our primary application area of psychophysics, the other benchmark problems are a low-dimensional ($d=2$) and a high-dimensional ($d=8$) synthetic function modeled after psychophysical discrimination tasks. The 2-d discrimination function is from \citet{owen2021}, and is linear in an intensity dimension ($x_2$) with slope given by a polynomial function of the other dimension ($x_1$). It is modeled after psychometric functions in domains such as haptics and multisensory perception. The 8-d discrimination function is similarly linear in an intensity dimension, with a slope given by a sum of shifted and scaled sinusoids, whose parameters form the other seven dimensions. Functional forms of both are given in the supplement.

Both discrimination test functions mimic psychophysics tasks in which the participant must identify which of two images/sounds/etc. has the stimulus, and we record if the identification was correct ($y=1$) or incorrect ($y=0$). When the stimulus intensity is very low, the participant must guess and the probability of a correct response is lower-bounded at $p=0.5$, and reaches this minimum along many of the edges of the search space. The goal in the experiment is to identify the detection threshold, where $p = 0.75$. 

We applied eight active sampling strategies to each of the three problems: the non-look-ahead straddle, applied to the posterior of the response probability $z$ as in Proposition \ref{prop:zpost}; localized look-ahead methods LocalSUR and LocalMI; global look-ahead methods GlobalSUR, GlobalMI, and EAVC; the approximate global SUR method of \citet{lyu21}, ApproxGlobalSUR; and quasi-random search with a scrambled Sobol sequence \citep{owen98}. To ensure differences are due solely to the acquisition function, all methods used the same GP classification surrogate model and the same gradient optimization of the acquisition function---see the supplement for details\footnote{Software for reproducing all of the methods and experiments in this paper, including the real-world task, is available at \url{https://github.com/facebookresearch/bernoulli_lse/}}. We evaluated performance using the Brier score \citep{brier}, a strictly proper scoring rule \citep{gneiting2007strictly} that assesses the quality and calibration of the level set posterior. See the supplement for extended results, including additional evaluation metrics, additional baseline methods such as BALD, and a sensitivity study.

Fig. \ref{fig:results-sim} shows the results of the benchmark experiments. In the 2-d discrimination problem, all LSE acquisition functions performed significantly better than the quasi-random baseline, and straddle and LocalMI were among the best-performing methods. However, on the high-dimensional problems, LocalMI performed worse than quasi-random search by a substantial margin, as did straddle and LocalSUR on the 8-d problem. Global methods generally outperformed localized methods and the quasi-random baseline, consistent with the conclusion that global look-ahead is a key ingredient needed to achieve consistently strong performance. Among global methods, GlobalSUR showed variable performance, consistent with the findings of \citet{lyu21} who noted that SUR underperformed with classification metamodels. Interestingly, ApproxGlobalSUR generally outperformed GlobalSUR, which seemed to underexplore on the 8-d discrimination problem, suggesting that the posterior approximations encouraged better exploration. GlobalMI and EAVC both performed consistently well across problems.

Consistent with Fig. \ref{fig:acquisitions}, we found that localized look-ahead methods sampled significantly more near the edges. On the Binarized Hartmann6 problem, 99\% of samples with LocalMI were near an edge, compared to 80\% with GlobalMI, 55\% with EAVC, and 47\% with quasi-random search---see the supplement for a full analysis of edge sampling behavior. On the low-dimensional problem, the tendency to oversample the edges was not as detrimental and the localized methods showed some advantage. However, they failed badly in high dimensions where a higher degree of exploration was critical.

Importantly, the wall time to select the next point with the global acquisition functions was generally under a second in a standard multi-core setting, making these methods suitable for real human experiments---see the supplement for details on running times.

\section{REAL PSYCHOPHYSICS TASK}\label{sec:realworld}
The contrast sensitivity function (CSF) describes how human visual sensitivity depends on stimulus properties such as spatial frequency and contrast. It is a crucial model of human vision used for clinical assessment \citep{clinicalcsf} and in applied settings to estimate visual appearance \citep{campbellandrobson, hdr-vdp2, fovvvdp}. Contrast sensitivity is affected by a number of variables including eccentricity, size, color, orientation, mean luminance, spatial frequency, and temporal frequency \citep{Robson66, Wright83, mullen85, FOLEY07, kim20}. Contrast sensitivity thresholds across these dimensions have been previously measured piecemeal with traditional psychophysical methods which cannot scale beyond three or four dimensions, and therefore a definitive CSF simultaneously accounting for all of these variables does not exist.

To evaluate our methods on CSF threshold identification, we ran a real CSF psychophysical discrimination study on one of the authors using Psychopy \citep{Peirce2019PsychoPy2EI}. As is standard for CSF measurement, stimuli were animated, Gaussian-windowed sinusoidal gratings, conventionally known as Gabor patches \citep{gabor1946theory}, generated by convolving a sinusoid with a Gaussian, which then had one half of the image scrambled, and were animated by advancing the phase of the sinusoid. Fig.~\ref{fig:real_stims} shows three examples of stimuli used in this experiment, where the task for the participant was to identify whether the scrambled half of the image was on the left or the right, and the $y$ response is whether they correctly identified the side with the stimulus ($y=1$) or not ($y=0$). This was thus a discrimination task like those in Section \ref{sec:experiments} where the success rate was lower-bounded by $p=0.5$ (guessing). The goal was to identify the detection threshold where $p=0.75$. For each stimulus, we varied eight properties that are known to affect the CSF: background luminance, stimulus contrast, stimulus orientation, temporal frequency of the animation, spatial frequency, stimulus size, and location on the screen (angle and distance from center). We collected responses to 1000 quasi-random stimuli which were used to fit a 6-d surrogate model for benchmarking purposes---see the supplement for details. We evaluated the same LSE methods from Section \ref{sec:experiments}, using the surrogate model as ground truth from which Bernoulli responses were simulated.

\begin{figure}[t]
    \includegraphics{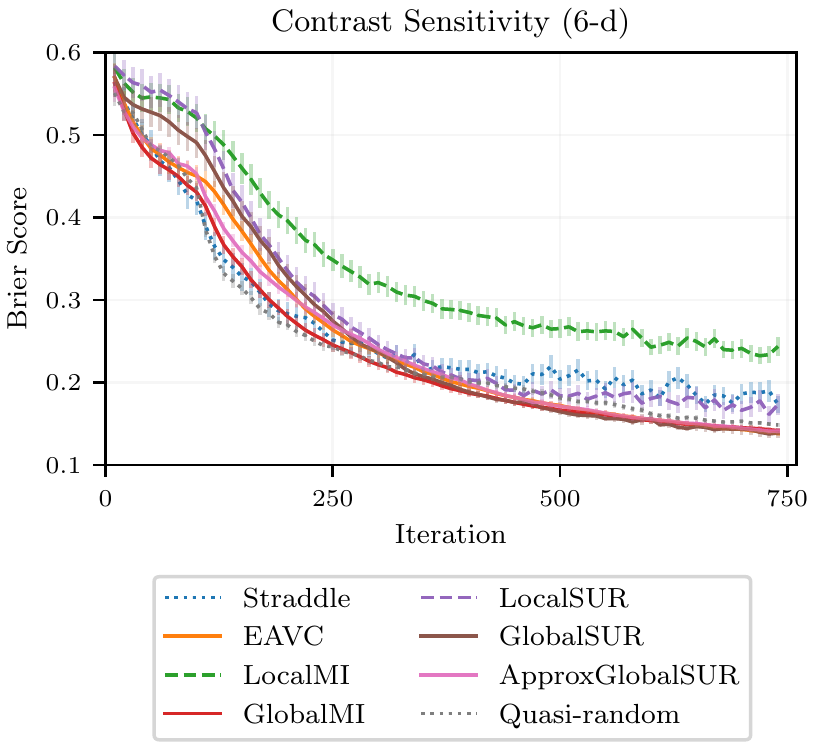}
    \caption{\textbf{Real psychophysics task results}. Performance averaged across 280 repeated runs on the real contrast sensitivity function task. As in the synthetic benchmarks, non-global acquisition functions generally performed poorly, and the quasi-random baseline was remarkably strong. Global look-ahead methods performed best, and significantly better than the other methods.}
    \label{fig:results-real}
\end{figure}

The results on the real-world CSF task are shown in Fig. \ref{fig:results-real}. As in the high-dimensional synthetic problems, straddle and the local methods (LocalMI and LocalSUR) failed to outperform the non-active quasi-random baseline. Global look-ahead methods continued to consistently perform best.

\section{DISCUSSION}\label{sec:discussion}

We have derived analytic formulae for the look-ahead level-set posteriors with Bernoulli observations, and used them to construct acquisition functions for Bernoulli LSE. The formulae enabled applying state-of-the-art approaches for Gaussian observations, SUR and MI, to the Bernoulli setting, while also making it easy to construct EAVC, a novel acquisition function. Prior to this work, none of the look-ahead acquisition functions could be applied directly to Bernoulli LSE, leaving the straddle and quasi-random search as the primary available strategies. The results of Theorem \ref{thm:post} have thus greatly expanded the Bernoulli LSE acquisition toolbox.

Our empirical results showed that the \textit{global} look-ahead acquisition functions developed in this paper are essential for consistently achieving good estimates of the level set. The localized look-ahead methods LocalSUR and LocalMI have previously been studied with Gaussian observations, and shown to perform well on those problems \citep{picheny10, nguyen21}; we found that in high-dimensional problems with Bernoulli observations, they often performed worse than quasi-random search. This highlights the importance of evaluating acquisition functions particularly for this setting, as well as the key differences between this setting and the more typical Gaussian observations.

The setting we consider here is particularly challenging for active learning because Bernoulli observations in essence have very high noise levels in the region of interest. The standard deviation of a Bernoulli random variable is $\sqrt{p(1-p)}$, which at the target threshold here of $p=0.75$ is approximately $0.43$; this nearly equals the total variation of the entire function ($p$ from $0.5$ to $1$). Noise levels on par with the total function variation make it difficult to learn a decent global surrogate, and mean that many observations are required to significantly reduce posterior variance. This exacerbates existing boundary over-exploration pathologies of GPs in general \citep{Siivola2018} and classification GPs in particular \citep[\textit{e.g.}][]{Song2017b}. Localized look-ahead methods target areas of high variance on the edges, and, unlike in the Gaussian case, get stuck on the edges because in high dimensions the variance is never sufficiently reduced. Thus having a look-ahead acquisition function is not by itself sufficient to achieve good performance, we must look ahead to the global impact of a point.

The strong performance of quasi-random search as a baseline on the high-dimensional problems highlights an interesting difference between LSE and BO, where quasi-random search does not typically provide as strong a comparator. In BO, the target is often a single point, the global optimum. In LSE, we are trying to learn the boundary of the set $\mathcal{L}_{\gamma}(f)$, which in general could be a $(d-1)$-dimensional manifold. LSE thus inherently requires more global evaluation than BO. Quasi-random search performs maximally global evaluation of the function, which is much less detrimental for LSE than for BO. Much of the literature on LSE has not included a critical evaluation of random or quasi-random search in benchmark experiments---our results show that one should always be included.

Our real-world application focused on a visual psychophysics task, but there is a broad set of other useful applications for Bernoulli LSE. In robotics, one may wish to find the set of controller parameters under which a robot can successfully traverse an obstacle with high probability \citep{tesch13}. This problem can be cast as Bernoulli LSE. Several other important classes of problems have discrete outputs, such as ordinal regression \citep{chu2005ordinal} and preference learning \citep{chu2005pref, furnkranz2010pref}. Finding all configurations that are preferred to a current baseline via preference learning can be cast as Bernoulli LSE. Finally, psychophysics itself comprises many application areas: it is a foundational component of AR/VR research, a rapidly developing area of computing, while also having several important applications in disease diagnostics and management, as described in the Introduction.

While our results show strong estimation performance with the acquisition functions developed here, there remain several important areas of future work. First, in acquisition function development: There is no single best acquisition function for all problems, and LSE as a field will benefit from expanding the acquisition toolbox. Our results make it easy to compute any look-ahead acquisition function that is a function of the level set posterior, which can accelerate development for Bernoulli LSE. Second, in the scope of the look-ahead: Recent work in BO has targeted multi-step acquisition functions, in which we look ahead multiple steps of acquisition rather than just one as is done here \citep{gonzalez2016glasses, jiang2020binoculars, jiang2020efficient}. Our results here could form a basis for non-myopic Bernoulli LSE. Finally, in the model classes: Similar results may exist for other types of classification GPs, such as the logit model or skew GPs \citep{skewgp}, which have different advantages relative to the probit model used here.

\subsubsection*{Acknowledgements}
Thanks to James Wilson for developing a differentiable BVN routine that enabled gradient optimization of the look-ahead acquisition functions; Lucy Owen for assistance with implementing the human experiments; the AEPsych team for development of the \href{https://github.com/facebookresearch/aepsych}{AEPsych} platform in which the experiments were run; and the anonymous reviewers for their constructive feedback.

\bibliography{refs}


\newpage
\onecolumn

\aistatstitle{Look-Ahead Acquisition Functions for Bernoulli Level Set Estimation: Supplementary Materials}

\setcounter{section}{0}
\setcounter{equation}{0}
\setcounter{figure}{0}
\setcounter{table}{0}
\setcounter{page}{1}
\renewcommand{\thesection}{S\arabic{section}}
\renewcommand{\theequation}{S\arabic{equation}}
\renewcommand{\thefigure}{S\arabic{figure}}
\renewcommand{\thetable}{S\arabic{table}}
\renewcommand{\thelemma}{S\arabic{lemma}}



\section{PROOFS}
Here we provide proofs of the results in Propositions \ref{prop:zpost} and \ref{prop:py1}, and Theorem \ref{thm:post}. \citet{owen_integrals} undertook the Herculean effort of producing a comprehensive collection of solutions to Gaussian integrals. We use several of his results, given in the following Lemma. We use the following notation for special functions:
\begin{itemize}
\setlength\itemsep{0em}
    \item $\Phi(\cdot)$ is the standard Gaussian cumulative distribution function.
    \item $\phi(\cdot)$ is the standard Gaussian density function.
    \item $\phi(\cdot ; \mu, \sigma^2)$ is the Gaussian density with mean $\mu$ and variance $\sigma^2$, so that $\phi(x ; \mu, \sigma^2) = \frac{1}{\sigma} \phi \left( \frac{x - \mu}{\sigma} \right)$.
    \item $T(\cdot, \cdot)$ is Owen's T function.
    \item $\textrm{BvN}(\cdot, \cdot; \rho)$ is the standard bivariate normal cumulative distribution function.
\end{itemize}

\begin{lemma}
\begin{align}\label{eq:int1}
    \int_{-\infty}^{+\infty} \Phi(a + bx) \phi(x) dx &= \Phi \left( \frac{a}{\sqrt{1 + b^2}} \right),\\\label{eq:int2}
    \int_{-\infty}^{+\infty} \Phi(a + bx)^2 \phi(x) dx &= \Phi \left( \frac{a}{\sqrt{1 + b^2}} \right) - 2 T \left( \frac{a}{\sqrt{1 + b^2}}, \frac{1}{\sqrt{1 + 2b^2}} \right),\\\label{eq:int3}
    \int_{-\infty}^{+\infty} \Phi(a + b x) \Phi(h + k x) \phi(x) dx &= BvN\left(\frac{a}{\sqrt{1 + b^2}}, \frac{h}{\sqrt{1 + k^2}}; \frac{bk}{\sqrt{1 + b^2}\sqrt{1 + k^2}} \right).
\end{align}
\end{lemma}
These results are 10,010.8, 20,010.4, and 20,010.3, respectively, from \citet{owen_integrals}.

Throughout the proofs in this section, for notational convenience and clarity we will use the shorthand $f_* = f(\mathbf{x}_*)$ to represent the latent function value at $\mathbf{x}_*$, and will let $\mu_* = \mu(\mathbf{x}_* | \mathcal{D}_n)$ and $\sigma^2_* = \sigma^2(\mathbf{x}_* | \mathcal{D}_n)$ indicate the posterior mean and variance of $f_*$. Thus, $f_* | \mathcal{D}_n \sim \mathcal{N}(\mu_*, \sigma^2_*)$. Table \ref{table:shorthand} provides a complete list of the abbreviated notation used throughout this supplement.

\begin{table}[tbh]
    \centering
    \begin{tabular}{c|c}
    \textbf{Short-hand notation} & \textbf{Definition}  \\\hline
       $f_*$, $f_q$  & $f(\mathbf{x}_*)$, $f(\mathbf{x}_q)$\\
       $\mu_*$, $\mu_q$& $\mu(\mathbf{x}_* | \mathcal{D}_n)$, $\mu(\mathbf{x}_q | \mathcal{D}_n)$\\
       $\sigma_*$, $\sigma_q$& $\sigma(\mathbf{x}_* | \mathcal{D}_n)$, $\sigma(\mathbf{x}_q | \mathcal{D}_n)$\\
       $\sigma_{q*}$ & $\textrm{Cov}[f(\mathbf{x}_q), f(\mathbf{x}_*) | \mathcal{D}_n]$\\
       $a_*$ & $\frac{\mu_*}{\sqrt{1 + \sigma_*^2}}$\\
       $c_*$ & $\frac{1}{\sqrt{1 + 2 \sigma_*^2}}$\\
       $b_*$, $b_q$ &  $\frac{\gamma - \mu_*}{\sigma_*}$, $\frac{\gamma - \mu_q}{\sigma_q}$ \\
       $Z_{q*}$ & $\textrm{BvN} \left( a_*  ,  b_q; \frac{- \sigma_{q*}}{\sigma_q \sqrt{1 + \sigma_*^2}} \right)$ \\
       $Z_{**}$ & $\textrm{BvN} \left( a_*  ,  b_*; \frac{- \sigma_{*}}{\sqrt{1 + \sigma_*^2}} \right)$
       
    \end{tabular}
    \caption{Abbreviated notation used throughout the proofs and other results in this supplement.}
    \label{table:shorthand}
\end{table}

\begin{proof}[Proof of Proposition \ref{prop:zpost}]

\begin{align}\nonumber
    \mathbb{E}[z(\bm{x}_*) | \mathcal{D}_n] &= \mathbb{E}[\Phi(f_*) | \mathcal{D}_n] \\\label{eq:prop1step0}
    &= \int_{-\infty}^{+\infty} \Phi(f_*) \phi( f_*;\mu_*, \sigma^2_*) df_*\\\nonumber\displaybreak[1]
    &= \frac{1}{\sigma_*} \int_{-\infty}^{+\infty} \Phi(f_*) \phi\left( \frac{f_* - \mu_*}{\sigma_*} \right) df_*\\\label{eq:prop1astep1}
    &= \int_{-\infty}^{+\infty} \Phi(\mu_* + \sigma_* \tilde{f}_*) \phi( \tilde{f}_* )  d\tilde{f}_*\\\label{eq:prop1astep2}
    &= \Phi \left( \frac{\mu_*}{\sqrt{1 + \sigma_* ^2}} \right),
\end{align}
where (\ref{eq:prop1astep1}) used the change of variables $\tilde{f}_* = \frac{f_* - \mu_*}{\sigma_*}$, and (\ref{eq:prop1astep2}) used (\ref{eq:int1}). For the variance,
\begin{equation*}
    \textrm{Var}[z(\bm{x}_*) | \mathcal{D}_n] = 
    \mathbb{E}[z(\bm{x}_*)^2 | \mathcal{D}_n] - (\mathbb{E}[z(\bm{x}_*) | \mathcal{D}_n])^2,
\end{equation*}
where, similarly as before,
\begin{align}\nonumber
    \mathbb{E}[z(\bm{x}_*)^2 | \mathcal{D}_n] &= \int_{-\infty}^{+\infty} \Phi(f_*)^2 p(f_* | \mathcal{D}_n) df_*\\\nonumber
    &= \int_{-\infty}^{+\infty} \Phi(\mu_* + \sigma_* \tilde{f}_*)^2 \phi( \tilde{f}_* )  d\tilde{f}_*\\\nonumber
    &= \Phi \left( \frac{\mu_*}{\sqrt{1 + \sigma_*^2}} \right) - 2 T \left( \frac{\mu_*}{\sqrt{1 + \sigma_*^2}}, \frac{1}{\sqrt{1 + 2\sigma_*^2}} \right)
\end{align}
using (\ref{eq:int2}). Letting $a_* = \frac{\mu_*}{\sqrt{1 + \sigma_*^2}}$ and $c_* = \frac{1}{\sqrt{1 + 2 \sigma_*^2}}$ as in Proposition \ref{prop:zpost}, we have that
\begin{equation*}
    \textrm{Var}[z(\bm{x}_*) | \mathcal{D}_n] = \Phi(a_*) - \Phi(a_*)^2 - 2T(a_*, c_*).
\end{equation*}
\end{proof}

\begin{proof}[Proof of Proposition \ref{prop:py1}]
\begin{align*}
    \mathbb{P}(y_* = 1 | \mathcal{D}_n, \mathbf{x}_*) = \int_{-\infty}^{+\infty} \mathbb{P}(y_* = 1 | f_*) p(f_* | \mathcal{D}_n) df_* = \int_{-\infty}^{+\infty} \Phi(f_*) \phi(f_* ; \mu_*, \sigma^2_*) df_*,
\end{align*}
which we have already seen in (\ref{eq:prop1step0}) equals $\Phi(a_*)$.
\end{proof}

For the proof of Theorem \ref{thm:post}, we will introduce additional shorthand notation  $f_q = f(\mathbf{x}_q)$, and as before will let $f_q | \mathcal{D}_n \sim \mathcal{N}(\mu_q, \sigma^2_q)$. We let $\sigma_{q*}$ denote the covariance between $f_*$ and $f_q$. We use the following result on the conditional distribution between $f_*$ and $f_q$.

\begin{lemma}\label{lemma:cond}
Let
\begin{equation*}
\begin{pmatrix}
f_1\\
f_2
\end{pmatrix} \sim \mathcal{N} \left( \begin{pmatrix}
\mu_1\\
\mu_2
\end{pmatrix}, \begin{pmatrix}
\sigma^2_1 & \sigma_{12}\\
\sigma_{21} & \sigma^2_2
\end{pmatrix}  \right).
\end{equation*}
Then, the conditional density for $f_1$ given $f_2 \leq \gamma$ is
\begin{equation*}
    p(f_1 | f_2 \leq \gamma) = \frac{\phi \left(\frac{f_1 - \mu_1}{\sigma_1} \right)}{\sigma_1 \Phi \left(\frac{\gamma - \mu_2}{\sigma_2} \right)} \Phi \left( \frac{\gamma - \mu_2 -  \frac{\sigma_{12}}{\sigma^2_1} (f_1 - \mu_1)}{\sqrt{\sigma_2^2 - \frac{\sigma^2_{12}}{\sigma^2_1}}} \right) 
\end{equation*}
\end{lemma}
\begin{proof}
By Bayes' theorem,
\begin{equation}\label{eq:lemma2step1}
p(f_1 | f_2 \leq \gamma) = \frac{p(f_2 \leq \gamma | f_1) p(f_1)}{p(f_2 \leq \gamma)} = \frac{p(f_2 \leq \gamma | f_1) \phi \left(\frac{f_1 - \mu_1}{\sigma_1} \right)}{\sigma_1 \Phi \left(\frac{\gamma - \mu_2}{\sigma_2} \right)}.
\end{equation}
It is well-known that
\begin{equation*}
p(f_2 | f_1 = x) = \mathcal{N}(\mu_{2|1}, \sigma_{2|1})
\end{equation*}
where $\mu_{2|1} = \mu_2 +  \frac{\sigma_{12}}{\sigma^2_1} (x - \mu_1)$ and $\sigma^2_{2|1} = \sigma_2^2 - \frac{\sigma_{12}^2}{\sigma^2_1}$. Thus,
\begin{equation*}
    p(f_2 \leq \gamma | f_1) = \Phi \left( \frac{\gamma - \mu_{2|1}}{\sigma_{2|1}} \right) = \Phi \left( \frac{\gamma - \mu_2 -  \frac{\sigma_{12}}{\sigma^2_1} (f_1 - \mu_1)}{\sqrt{\sigma_2^2 - \frac{\sigma^2_{12}}{\sigma^2_1}}} \right),
\end{equation*}
which when plugged into (\ref{eq:lemma2step1}) produces the result.
\end{proof}

\begin{proof}[Proof of Theorem \ref{thm:post}]
By Bayes' theorem, we have that
\begin{align*}
    \pi(\mathbf{x}_q | \mathcal{D}_{n+1}(\mathbf{x}_*, y_* = 1)) &= \mathbb{P}(f_q \leq \gamma | \mathcal{D}_n, \mathbf{x}_*, y_* = 1)\\
    &= \frac{\mathbb{P}(y_* = 1 | \mathcal{D}_n, \mathbf{x}_*, f_q \leq \gamma) \mathbb{P}(f_q \leq \gamma| \mathcal{D}_n)}{\mathbb{P}(y_* = 1 | \mathcal{D}_n, \mathbf{x}_*)}.
\end{align*}
From Proposition \ref{prop:py1} we know the denominator equals $\Phi(a_*)$ and can easily compute $\mathbb{P}(f_q \leq \gamma| \mathcal{D}_n) = \Phi \left( \frac{\gamma - \mu_q}{\sigma_q} \right) = \Phi(b_q)$, so the only term remaining is $\mathbb{P}(y_* = 1 | \mathcal{D}_n, \mathbf{x}_*, f_q \leq \gamma)$.

\begin{align}\nonumber
    \mathbb{P}(y_* = 1 | \mathcal{D}_n, \mathbf{x}_*, f_q \leq \gamma) &= \int_{-\infty}^{+\infty} \mathbb{P}(y_* = 1 | f_*) p(f_* | \mathcal{D}_n, f_q \leq \gamma) df_* \\\nonumber
    &= \int_{-\infty}^{+\infty} \Phi(f_*) p(f_* | \mathcal{D}_n, f_q \leq \gamma) df_*\\\label{eq:thm1proofstep1}
    &= \frac{1}{\sigma_* \Phi\left( \frac{\gamma - \mu_q}{\sigma_q} \right) } \int_{-\infty}^{+\infty} \Phi(f_*)   \Phi \left( \frac{\gamma - \mu_q -  \frac{\sigma_{q*}}{\sigma^2_*} (f_* - \mu_*)}{\sqrt{\sigma_q^2 - \frac{\sigma^2_{q*}}{\sigma^2_*}}} \right) \phi \left(\frac{f_* - \mu_*}{\sigma_*} \right) df_*\\\nonumber
    &= \frac{1}{\Phi\left( \frac{\gamma - \mu_q}{\sigma_q} \right) } \int_{-\infty}^{+\infty} \Phi(\mu_* + \sigma_* \tilde{f}_*)   \Phi \left( \frac{\gamma - \mu_q -  \frac{\sigma_{q*}}{\sigma_*} \tilde{f}_* }{\sqrt{\sigma_q^2 - \frac{\sigma^2_{q*}}{\sigma^2_*}}} \right) \phi (\tilde{f}_*) d\tilde{f}_*\\\label{eq:thm1proofstep2}
    &= \frac{1}{\Phi\left( \frac{\gamma - \mu_q}{\sigma_q} \right) }  \textrm{BvN}\left(\frac{\mu_*}{\sqrt{1 + \sigma_*^2}}, \frac{\gamma - \mu_q}{\sqrt{\sigma_q^2 - \frac{\sigma^2_{q*}}{\sigma^2_*}}\sqrt{1 +  \frac{\sigma^2_{q*}} {\sigma_*^2 \sigma_q^2 - \sigma^2_{q*}}  }}; \frac{ \frac{-\sigma_* \sigma_{q*}}{\sqrt{\sigma_*^2 \sigma_q^2 - \sigma^2_{q*}}}}{\sqrt{1 + \sigma_*^2}\sqrt{1 +  \frac{\sigma^2_{q*}} {\sigma_*^2 \sigma_q^2 - \sigma^2_{q*}}  }} \right)\\\nonumber
    &= \frac{1}{\Phi\left( \frac{\gamma - \mu_q}{\sigma_q} \right) }  \textrm{BvN}\left(\frac{\mu_*}{\sqrt{1 + \sigma_*^2}}, \frac{\gamma - \mu_q}{\sigma_q}; \frac{-\sigma_{q*}}{ \sigma_q \sqrt{1 + \sigma_*^2}} \right).
\end{align}
Here (\ref{eq:thm1proofstep1}) used Lemma \ref{lemma:cond}, and (\ref{eq:thm1proofstep2}) used (\ref{eq:int3}) with $a=\mu_*$, $b=\sigma_*$, $h = \frac{\gamma - \mu_q}{\sqrt{\sigma_q^2 - \frac{\sigma^2_{q*}}{\sigma^2_*}}}$, and $k = \frac{-\sigma_{q*}}{\sqrt{\sigma_*^2 \sigma_q^2 - \sigma^2_{q*}}}$. Combining this term with the other terms, and using the convenient definitions of $a_*$ and $b_q$, we have that

\begin{equation*}
    \pi(\mathbf{x}_q | \mathcal{D}_{n+1}(\mathbf{x}_*, y_* = 1)) = \frac{1}{\Phi(a_*)} \textrm{BvN}\left(a_*, b_q; \frac{-\sigma_{q*}}{ \sigma_q \sqrt{1 + \sigma_*^2}} \right).
\end{equation*}
For the $y_*=0$, case, 
\begin{equation*}
    \pi(\mathbf{x}_q | \mathcal{D}_{n+1}(\mathbf{x}_*, y_* = 0)) = \frac{\mathbb{P}(y_* = 0 | \mathcal{D}_n, \mathbf{x}_*, f_q \leq \gamma) \mathbb{P}(f_q \leq \gamma| \mathcal{D}_n)}{\mathbb{P}(y_* = 0 | \mathcal{D}_n, \mathbf{x}_*)}
\end{equation*}
The terms are easily computed from what we have already found: $\mathbb{P}(y_* = 0 | \mathcal{D}_n, \mathbf{x}_*) = 1 - \Phi(a_*)$, and $\mathbb{P}(y_* = 0 | \mathcal{D}_n, \mathbf{x}_*, f_q \leq \gamma) = 1 - \frac{1}{\Phi(b_q)} \textrm{BvN}\left(a_*, b_q; \frac{-\sigma_{q*}}{ \sigma_q \sqrt{1 + \sigma_*^2}} \right)$. Plugging these in yields the result in the Theorem.
\end{proof}

\section{ACQUISITION EXPRESSIONS}
Here we provide the full expression used to compute each look-ahead acquisition function, using the same posterior short-hand notation as in the previous section.

\paragraph{Global SUR}
\begin{align*}
    \alpha&_\textrm{GlobalSUR}(\mathbf{x}_*) = \\
    &\sum_{\mathbf{x}_q \in \mathcal{G}}  \left( \min(\Phi(b_q), 1 - \Phi(b_q)) - \min \left( Z_{q*}, \Phi(a_*) - Z_{q*}\right) -  \min \left( \Phi(b_q) - Z_{q*}, \Phi(-a_*) - \Phi(b_q) - Z_{q*} \right) \right). 
\end{align*}

\paragraph{Localized SUR}
\begin{equation*}
    \alpha_\textrm{LocalSUR}(\mathbf{x}_*) =  \min(\Phi(b_*), 1 - \Phi(b_*)) - \min \left( Z_{**}, \Phi(a_*) - Z_{**}\right) -  \min \left( \Phi(b_*) - Z_{**}, \Phi(-a_*) - \Phi(b_*) - Z_{**} \right) . 
\end{equation*}

\paragraph{Localized MI}
\begin{equation*}
    \alpha_\textrm{LocalMI}(\mathbf{x}_*) = H_b(\Phi(b_*)) - \Phi(a_*) H_b \left( \frac{Z_{**}}{\Phi \left(a_* \right)} \right) + \Phi(-a_*) H_b \left( \frac{\Phi\left( b_* \right) - Z_{**}}{\Phi \left(-a_* \right)} \right).
\end{equation*}

\paragraph{Global MI}
\begin{equation*}
    \alpha_\textrm{GlobalMI}(\mathbf{x}_*) = \sum_{\mathbf{x}_q \in \mathcal{G}}  \left( H_b(\Phi(b_q)) - \Phi(a_*) H_b \left( \frac{Z_{q*}}{\Phi \left(a_* \right)} \right) + \Phi(-a_*) H_b \left( \frac{\Phi\left( b_q \right) - Z_{q*}}{\Phi \left(-a_* \right)} \right) \right).
\end{equation*}

\paragraph{EAVC}
\begin{equation*}
    \alpha_\textrm{EAVC}(\mathbf{x}_*) = \Phi(a_*) \left| \sum_{\mathbf{x}_q \in \mathcal{G}}  \left( \Phi(b_q) - \frac{Z_{q*}}{\Phi \left(a_* \right)}  \right)  \right| + \Phi(-a_*) \left| \sum_{\mathbf{x}_q \in \mathcal{G}}  \left( \Phi(b_q) - \frac{\Phi\left( b_q \right) - Z_{q*}}{\Phi \left(-a_* \right)}  \right)  \right|.
\end{equation*}

\section{ADDITIONAL BENCHMARK EXPERIMENT RESULTS}

\subsection{Synthetic Functions}

The synthetic functions were designed to explore a variety of input and output patterns that are present in real LSE problems, and in psychophysics problems in particular. A common experimental paradigm in psychophysics is the \textit{two-alternative forced choice} (2AFC) method in which the participant is given two options and forced to select one. The CSF study in Section \ref{sec:realworld}, and illustrated in Fig. \ref{fig:real_stims}, is an example of a 2AFC task. For 2AFC tasks, the minimum probability of being correct is 0.5, because participants are forced to make a choice and in the absence of a detectable stimulus will guess randomly. Thus the probability output space is $[0.5, 1]$, and the goal in these experiments is typically to find the $\theta=0.75$ threshold, as is done in our experiment. However, there are other experimental designs, and other Bernoulli LSE tasks, in which the probability of success will vary from $0$ to $1$, and so to show that the methods are not limited to the 2AFC setting, we designed the Binarized Hartmann6 function to have probabilities from $0$ to $1$, and there set the target threshold to $\theta=0.5$. We now give the functional form for each synthetic function.

\subsubsection{Binarized Hartmann6 Function}
The Binarized Hartmann6 function was a binarization of the modified Hartmann 6-d function used by \citet{lyu21}. Their modified Hartmann 6-d function is:
\begin{equation*}
h(\mathbf{x}) =  1 - \sum_{i=1}^4 \alpha_i \exp \left(-\sum_{j=1}^6 A_{ij}(x_j - P_{ij})^2 \right)
\end{equation*}
with $\alpha = [2.0, 2.2, 2.8, 3.0]$,
\begin{equation*}
A = 
\begin{pmatrix}
8 & 3 & 10 & 3.5 & 1.7 & 6 \\
0.5 & 8 & 10 & 1.0 & 6 & 9 \\
3 & 3.5 & 1.7 & 8 & 10 & 6 \\
10 & 6 & 0.5 & 8 & 1.0 & 9
\end{pmatrix}, \textrm{ and }
P = 10^{-4} 
\begin{pmatrix}
1312 & 1696 & 5569 & 124 & 8283 & 5886 \\
2329 & 4135 & 8307 & 3736 & 1004 & 9991 \\
2348 & 1451 & 3522 & 2883 & 3047 & 6650 \\
4047 & 8828 & 8732 & 5743 & 1091 & 381
\end{pmatrix}.
\end{equation*}
We used $f(\mathbf{x}) = 3h(\mathbf{x}) - 2$ as the ground-truth latent function, so that Bernoulli samples were simulated according to $\Phi(f(\mathbf{x}))$. The input space for this problem is as in the classic Hartmann6 problem, $\mathbf{x} \in [0, 1]^6$. The output probabilities, $\Phi(f(\mathbf{x}))$, span $[0, 1]$, so for this problem the target threshold was set as $\theta = 0.5$.

\subsubsection{Psychophysical Discrimination, 2-d}
The latent function for this problem is computed as
\begin{equation*}
    f(x_1, x_2) = \frac{1 + x_2}{0.05 + 0.4 x_1^2 (0.2 x_1 - 1)^2}.
\end{equation*}
The input domain is $x_1, x_2 \in [-1, 1]$, and the output probabilities span $[0.5, 1]$, with the target threshold $\theta=0.75$.

\subsubsection{Psychophysical Discrimination, 8-d}
For $\mathbf{x} = [x_1, \ldots, x_8]$, we define
\begin{equation*}
c(\mathbf{x}) = \left(\frac{x_3}{2} \left(1 - \cos\left(\frac{3}{5} \pi x_2 x_8  + x_7 \right) \right) + x_4\right) \left(2 - x_6 \left( 1 + \sin\left(\frac{3}{10}\pi x_2 x_8 + x_7\right) \right) \right) - 1.
\end{equation*}
Then, the Bernoulli probability for the 8-d Psychophysical Discrimination function is computed as
\begin{equation*}
    z(\mathbf{x}) = \frac{1}{2}  + \frac{1}{2} \Phi \left( \frac{x_1 - c(\mathbf{x})}{x_5(2 + c(\mathbf{x}))} \right). 
\end{equation*}
The input space is $\mathbf{x} \in [-1, 1]^8$ and, as in the 2-d discrimination function, the output probabilities span [0.5, 1], so the target threshold was set to 0.75.

\subsection{Surrogate Model and Acquisition Optimization}

All methods and all experiments used the same surrogate model: A typical variational classification GP \citep{gpclass_vi} with 100 inducing points and an RBF kernel. Kernel hyperpriors were taken as the defaults from the Botorch package \citep{botorch}. Inducing points were selected by applying k-means to the observations. In each iteration of active learning, the model was updated with the new observation by refitting the variational distribution and kernel hyperparameters. In most iterations, the refitting was warm-started by beginning the fitting at the previous values. To avoid getting stuck in a local optimum, and as is common in Bayesian optimization, every 10\textsuperscript{th} iteration the re-fitting was done from scratch with a refreshed set of inducing points.

To avoid conflating acquisition quality with the ability to optimize the acquisition function, all acquisition functions were optimized in the same manner, using the gradient-based acquisition optimization utilities from the Botorch package. For the global acquisition functions, the reference set $\mathcal{G}$ was taken as a quasi-random (scrambled Sobol) set of 500 points, which was changed for each iteration. Straddle in the probability space was evaluated as a Monte Carlo acquisition function \citep{botorch} due to the lack of a differentiable implementation of Owen's T function.

Each benchmark run initialized with an initial design of 10 quasi-random points, after which the surrogate model was fit and all subsequent iterations used active sampling with the specified acquisition function. Throughout the active sampling, performance metrics were computed using a quasi-random test-set of 1000 points, which was sampled independently from anything done for the modeling or acquisition optimization.

\subsection{Additional Evaluation Metrics}

The results in the main text showed performance evaluated using Brier score, a strictly proper scoring rule. Proper scoring rules are an appropriate evaluation metric for this problem space because they assess not only the quality of the model point prediction, but also the calibration of posterior uncertainty. An alternative metric for evaluating classifiers in particular is the expected classification error, defined as $p(1-y) + (1-p)y$ for a classifier that provides $p$ as the probability that $y=1$ (\textit{i.e.}, that a point is below threshold), and $y$ the actual outcome (\textit{i.e.}, if the point was actually below threshold). Fig. \ref{fig:classerr} show the results of the benchmark experiments when evaluated using expected classification error. The conclusions of the experiments do not change under this alternative evaluation metric.

\begin{figure*}[tb]
    \includegraphics{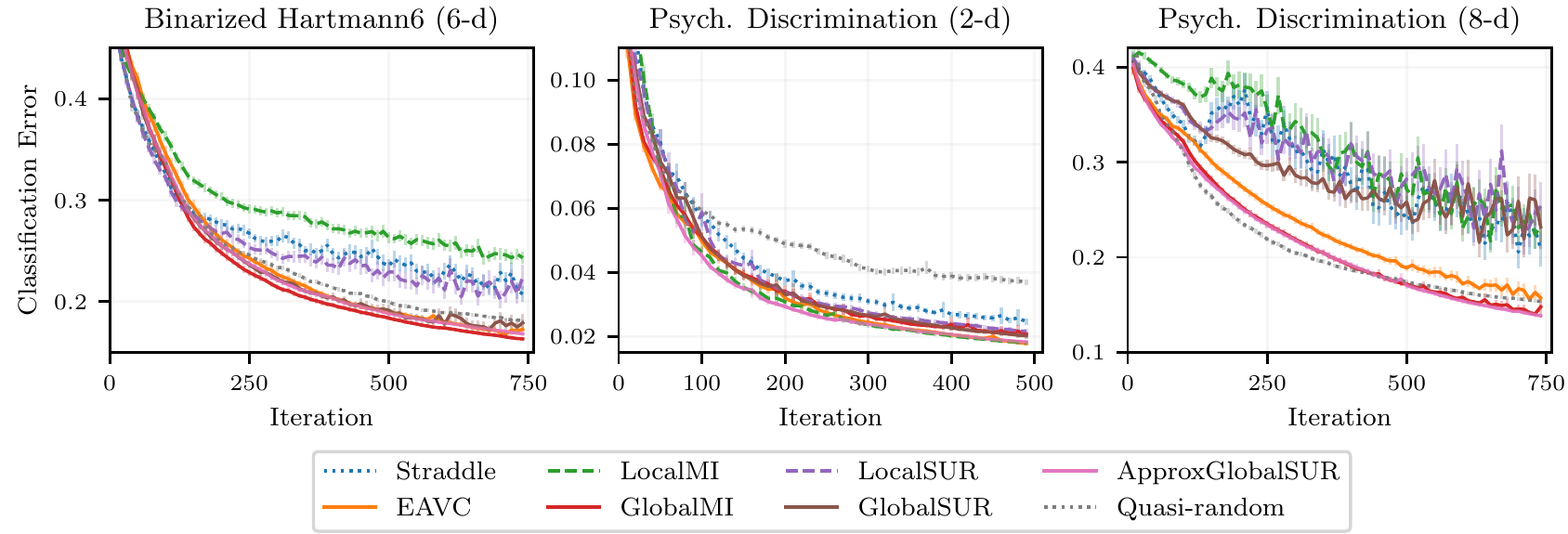}
    \vspace{-15pt}
    \caption{\textbf{Benchmark results: alternative metric}. Expected classification error (lower is better) of the GP surrogate model as a function of active sampling iteration, for the same benchmark results as in Fig. \ref{fig:results-sim}. As before, shown is the mean and two standard errors over 200 replications. On the high-dimensional problems, the localized look-ahead methods (LocalMI, LocalSUR) performed significantly worse than the quasi-random baseline. Global methods GlobalMI and EAVC were consistently the best.
    }
    \label{fig:classerr}
\end{figure*}

Fig. \ref{fig:gentime} shows the amount of wall time required to optimize the acquisition function for a model fit to 250 observations. Global acquisition functions required more wall-time to compute due to the global reference set $\mathcal{G}$, however with a per-iteration time between half a second and a second for the complete optimization when given multiple threads, they were well within the speed required for experiments with human participants. 

\begin{figure*}[tb]
    \includegraphics{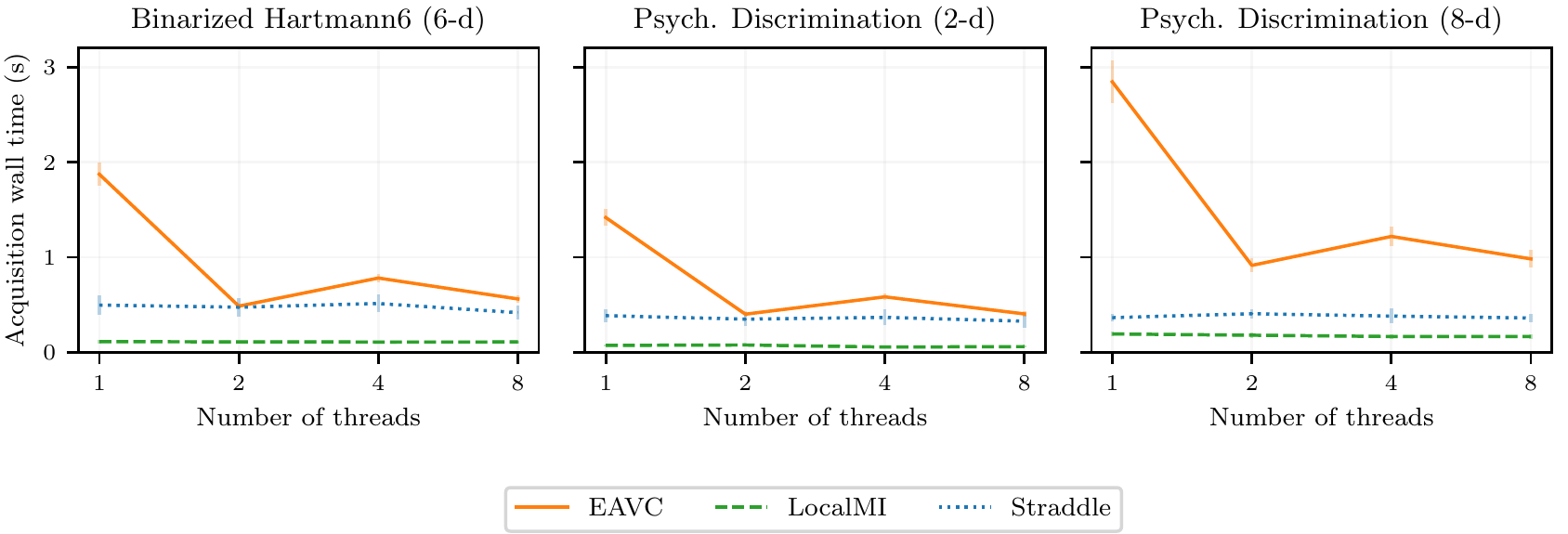}
    \vspace{-10pt}
    \caption{\textbf{Benchmark results: wall time}. Wall time required for acquisition optimization, in seconds, based on a surrogate model fit to 250 observations drawn from a Sobol sequence. This evaluation was done on an AWS EC2 \texttt{c6l.18xlarge} instance and is the average of 20 replications. Global acquisition methods required less than a second per iteration to identify the next point for evaluation.
    }
    \label{fig:gentime}
\end{figure*}

\begin{figure*}[tb]
    \includegraphics{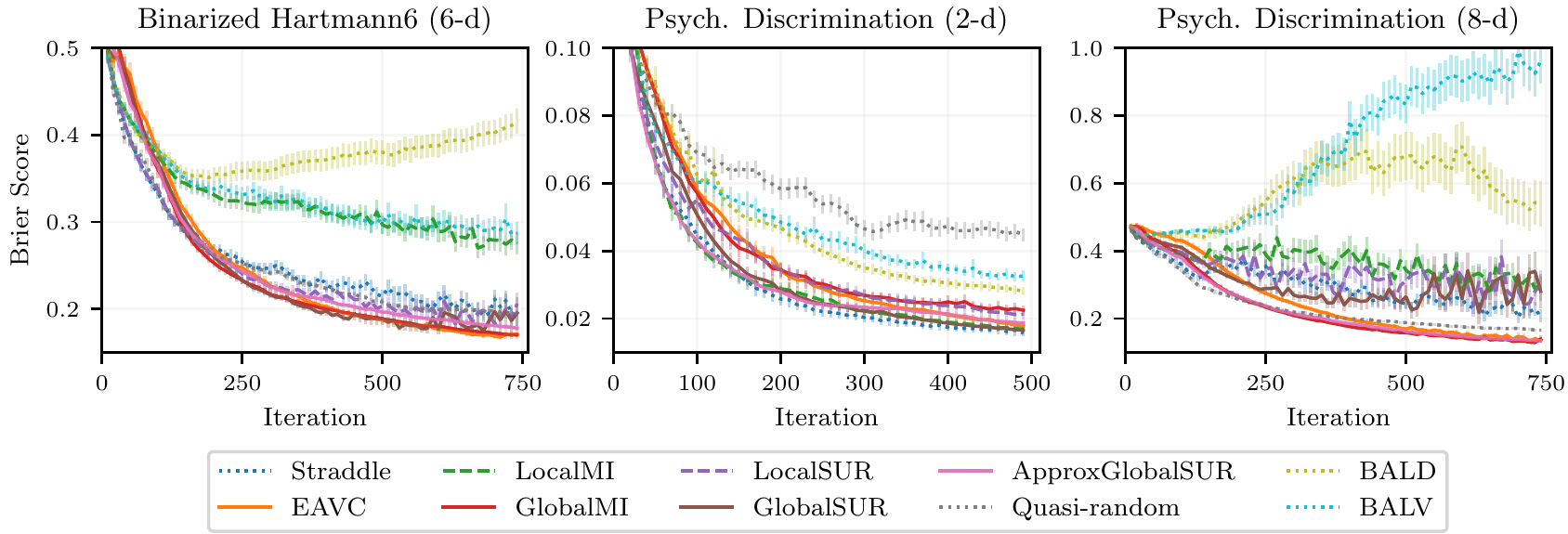}
    \vspace{-10pt}
    \caption{\textbf{Benchmark results: BALD}. The same benchmark results as in Fig. \ref{fig:results-sim}, with the addition of global active sampling methods BALD and BALV. Global active sampling methods can waste samples reducing uncertainty in areas far from the threshold, and here they performed worse than quasi-random search for LSE.
    }
    \label{fig:bald}
\end{figure*}

\subsection{Comparison to BALD and BALV}

The main text discusses how the use of global active sampling methods such as BALD \citep{Houlsby2011} or Bayesian active learning by variance (BALV) \citep{song2015} can be inefficient for level-set estimation because they may focus sampling effort on reducing variance in areas that are not close to the threshold. Fig. \ref{fig:bald} shows empirically that this is the case, by evaluating BALD and BALV on the same benchmark problems used in the main text. For LSE in high dimensions, BALD and BALV performed significantly worse than quasi-random search, which further emphasizes the importance of developing acquisition functions specifically for LSE.

\begin{figure*}[tb]
    \includegraphics{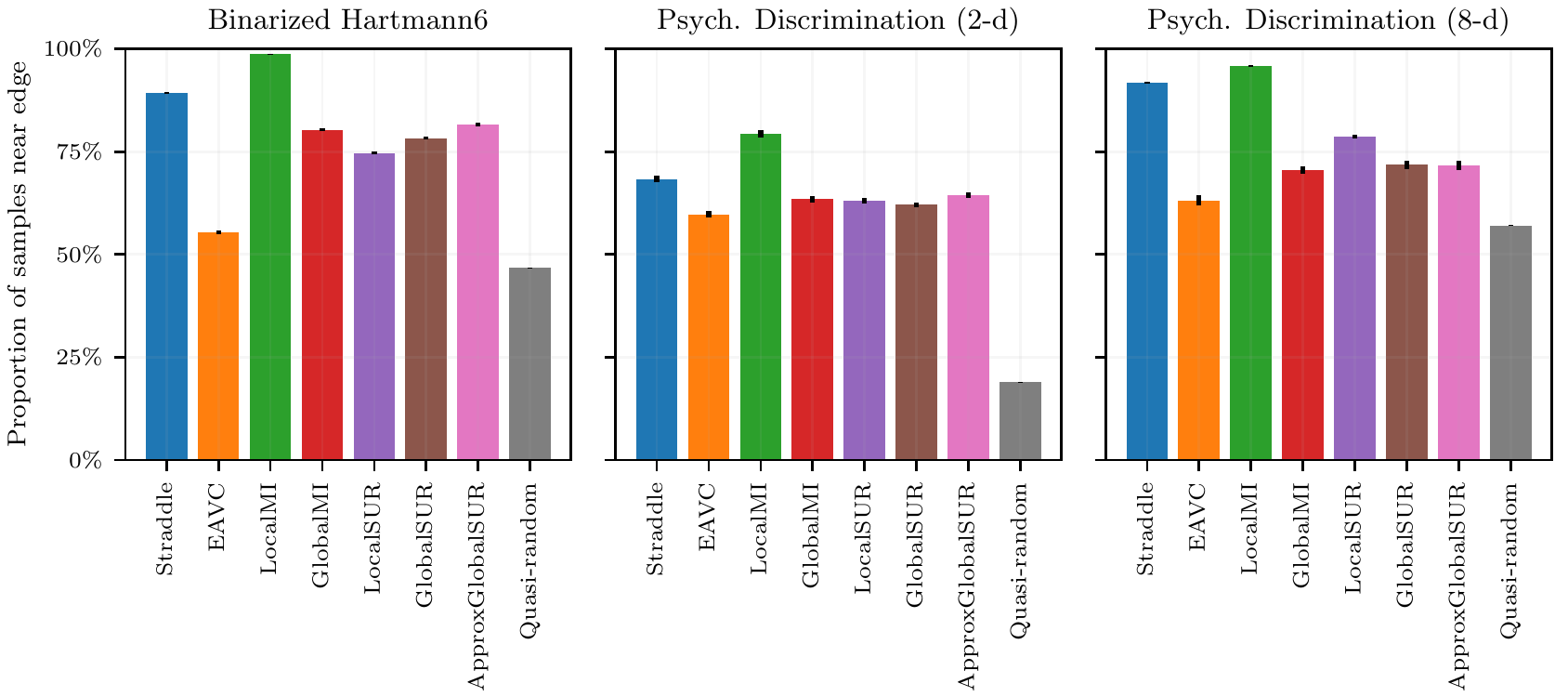}
    \caption{\textbf{Benchmark results: edge sampling}. For the benchmark results from Fig. \ref{fig:results-sim}, the proportion of active learning iterations in which the evaluated point was within 5\% of an edge of the search space. Like the Straddle acquisition, localized look-ahead methods did significantly more edge sampling than global look-ahead methods in high-dimensions.
    }
    \label{fig:edgesampling}
\end{figure*}

\subsection{An Analysis of Edge Sampling Behavior}

We highlighted in the main text that a source of poor performance for localized look-ahead methods is their tendency to oversample edge locations. This behavior is shown empirically in Fig. \ref{fig:edgesampling}, using the benchmark results from Section \ref{sec:experiments}. For each benchmark run, we evaluated the proportion of active learning samples that were within 5\% of the search space range of an edge. For instance, on the Binarized Hartmann6 problem where the domain is $[0, 1]^6$, this was the proportion of points with an element less than $0.05$ or greater than $0.95$. In high dimensions, the localized look-ahead methods LocalMI and LocalSUR, along with the straddle acquisition, sampled significantly more edge locations than the global look-ahead methods, or quasi-random sampling. LocalMI was particularly focused on the edges, with 99\% edge samples for Binarized Hartmann6. EAVC had the least tendency to sample edges, and in high dimensions had an edge sampling rate comparable to quasi-random search.

\begin{figure*}[tbh]
    \includegraphics{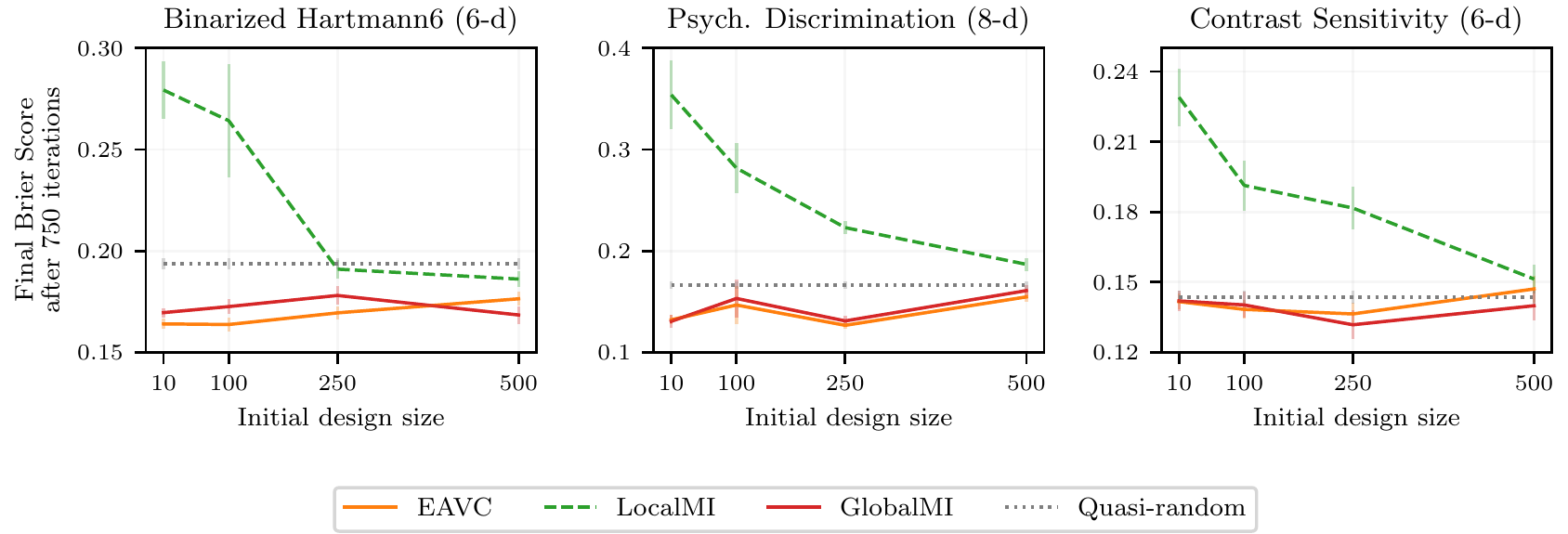}
    \vspace{-15pt}
    \caption{\textbf{Sensitivity study: initial design}. The Brier score (mean and two standard errors) after 750 total iterations (initial design plus active sampling) as a function of the size of the initial design. Global methods performed better when given more iterations of active sampling (smaller initialization). LocalMI benefited from a larger initialization, but never achieved the best performance of global methods on the Binarized Hartmann6 and Psychophysical Discrimination (8-d) problems.
    }
    \label{fig:initial_sens}
\end{figure*}
\begin{figure*}[tbh]
    \includegraphics{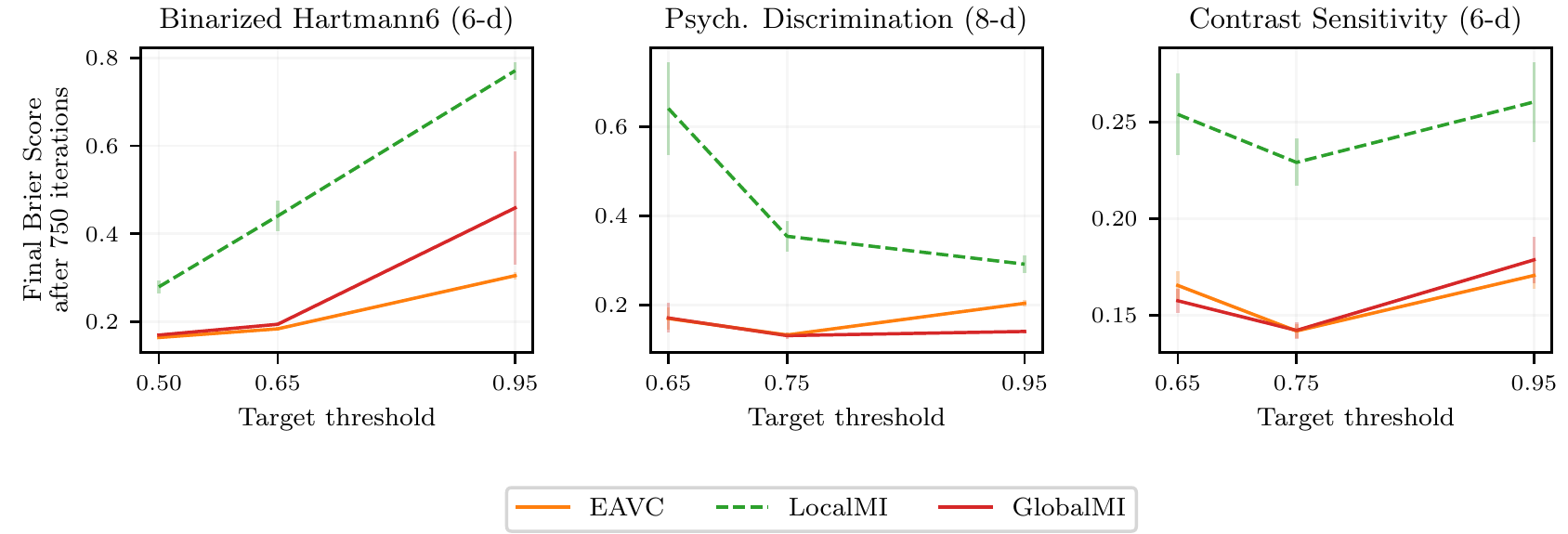}
    \vspace{-15pt}
    \caption{\textbf{Sensitivity study: target threshold}. The Brier score (mean and two standard errors) after 750 iterations when changing the target threshold. Changing the target threshold effectively changes the LSE problem. While some target thresholds pose more challenging tasks than others, across this range of settings the global look-ahead methods continued to generally perform best.
    }
    \label{fig:thresh_sens}
\end{figure*}

\subsection{Sensitivity Study}

The look-ahead acquisition functions described and developed in this paper do not have any hyperparameters that must be tuned. The straddle acquisition function has the $\beta$ hyperparameter, and \citet{lyu21} have done a sensitivity study of various strategies for selecting $\beta$ and found that none consistently performed well. Here we study sensitivity to two aspects of the experiments: the initial design, and the target threshold.

\subsubsection{Initial Design Sensitivity}
Each benchmark run in the results of Section \ref{sec:experiments} was initialized with 10 quasi-random points. We saw in Fig. \ref{fig:edgesampling} that localized look-ahead methods over-sampled edge locations, and hypothesized that a larger initial design would provide a better initial global surrogate model, under which the high degree of exploitation in the localized look-ahead methods could actually be beneficial.

Fig. \ref{fig:initial_sens} shows the final Brier score after 750 total iterations, for increasingly large initial designs. Note that the total number of iterations was fixed at 750, so that the initial design of size 10 had 740 active sampling iterations, while that of size 500 had only 250 active sampling iterations. The sensitivity study focuses on a subset of methods (EAVC, LocalMI, and GlobalMI) that were most characteristic. As hypothesized, LocalMI benefited significantly from having a larger initial design, and the improved global surrogate that a larger initial design entails. Particularly on the Binarized Hartmann6 problem, LocalMI performed significantly worse than quasi-random sampling with the small initial design of 10 points, but with an initial design of 250 (out of 750) points, was able to do slightly better than quasirandom.  While LocalMI did perform better with a larger initial design on all three high-dimensional problems, it still did not match the performance of the global methods. The global methods GlobalMI and EAVC, in contrast, performed best for the smallest initial designs, which permitted the most active sampling. Generally, performance with the global methods was robust to the size of the initial design up to 250 points (one third of the total budget).

\subsubsection{Target Threshold Sensitivity}

Fig. \ref{fig:thresh_sens} shows how the final Brier score varies with the target threshold for the problem. Changing the target threshold significantly alters the LSE problem, by focusing the active sampling in a different part of the search space. LSE performance thus changes when the target threshold is changed, however Fig. \ref{fig:thresh_sens} shows that across this large set of target level sets, the global look-ahead methods continued to consistently be the best.

\section{DETAILS OF REAL-WORLD EXPERIMENT}
For the real-world CSF experiment, the stimulus feature space was 8-dimensional and we collected 1000 stimuli generated from a Sobol sequence over that stimulus space. We used a GP classification surrogate model as ground truth. The model was fit using an RBF kernel over six of the stimulus features to create a 6-d problem space. Two stimulus properties, angular dimensions of eccentricity and orientation, were left unmodeled. This effectively added noise to the surrogate function and increased the difficulty of level-set estimation. The model was a variational classification GP with training locations used for inducing points. The GP mean was used for the ground-truth latent $f$ from which Bernoulli responses were simulated.


\end{document}